%% file: main.tex
\documentclass[10pt,twocolumn,letterpaper]{article}

\usepackage[pagenumbers]{cvpr}     

\usepackage{graphicx}
\usepackage{amsmath}
\usepackage{amssymb}
\usepackage{booktabs}

\usepackage[pagebackref,breaklinks,colorlinks]{hyperref}

\usepackage[capitalize]{cleveref}
\crefname{section}{Sec.}{Secs.}
\Crefname{section}{Section}{Sections}
\Crefname{table}{Table}{Tables}
\crefname{table}{Tab.}{Tabs.}

\usepackage{amsmath, amssymb, amsthm, amsfonts}  
\usepackage{mathrsfs}
\usepackage{soul} 
\usepackage{url}            
\usepackage{microtype}      
\usepackage{xcolor}         
\usepackage{appendix}       

\newcommand\numberthis{\addtocounter{equation}{1}\tag{\theequation}} 

\DeclareMathOperator{\E}{\mathbb{E}}

\newtheorem{definition}{Definition}[section]
\newtheorem{proposition}{Proposition}

\newcommand{\EGzero}{e_0}
\newcommand{\EGone}{e_1}
\newcommand{\EGtwo}{e_2}
\newcommand{\EGthree}{e_3}

\newcommand{\DIzero}{d_0}
\newcommand{\DIone}{d_1}
\newcommand{\DItwo}{d_2}

\input{useful-commands}

\begin{document}

\title{Failure Modes of Domain Generalization Algorithms}

\author{Tigran Galstyan$^{1, 3}$, Hrayr Harutyunyan$^2$, Hrant Khachatrian$^{1,4}$, Greg Ver Steeg$^2$, Aram Galstyan$^2$\\
{\normalsize $^1$ YerevaNN, \ \ $^2$ USC Information Sciences Institute, $^3$ Russian-Armenian University, $^4$ Yerevan State University}
}
\maketitle

\begin{abstract}
\input{sections/abstract}
\end{abstract}

\input{sections/intro}

\input{sections/related}

\input{sections/problem-setting}

\input{sections/failure-modes}

\input{sections/experiments}

\input{sections/discussion}

\input{sections/conclusion}

{\small
\bibliographystyle{ieee_fullname}
\bibliography{main}
}

\appendix
\clearpage
\onecolumn
\include{sections/appendix}

\end{document}

%% file: useful-commands.tex
\newcommand{\rbr}[1]{\left(#1\right)}
\newcommand{\sbr}[1]{\left[#1\right]}
\newcommand{\cbr}[1]{\left\{#1\right\}}

\newcommand{\ignore}[1]{}

%% file: sections/abstract.tex
Domain generalization algorithms use training data from multiple domains to learn models that generalize well to unseen domains. 
While recently proposed benchmarks demonstrate that most of the existing algorithms do not outperform simple baselines, the established evaluation methods fail to expose the impact of various factors that contribute to the poor performance.
In this paper we propose an evaluation framework for domain generalization algorithms that allows decomposition of the error into components capturing distinct aspects of generalization.
Inspired by the prevalence of algorithms based on the idea of domain-invariant representation learning, we extend the evaluation framework to capture various types of failures in achieving invariance.
We show that the largest contributor to the generalization error varies across methods, datasets, regularization strengths and even training lengths.
We observe two problems associated with the strategy of learning domain-invariant representations.
On Colored MNIST, most domain generalization algorithms fail because they reach domain-invariance only on the training domains.
On Camelyon-17, domain-invariance degrades the quality of representations on unseen domains.
We hypothesize that focusing instead on tuning the classifier on top of a rich representation can be a promising direction.

%% file: sections/intro.tex
\section{Introduction}
\label{introduction}
Over the past decade machine learning research was predominantly focused on settings where the learner observes training data from an unknown distribution and is evaluated on testing data, sampled from the same distribution.
While modern deep learning approaches excel in such settings, they do significantly worse when the test data comes from a different distribution~\cite{torralba2011unbiased, pmlr-v97-recht19a}.
These methods might rely on dataset biases to perform well, and fail when those biases are eliminated~\cite{beery2018recognition, geirhos2020shortcut}.

The goal of generalizing beyond training distribution is formulated in the \emph{domain generalization (DG)} task, where the learner observes training data from multiple domains and is evaluated on unseen domains.
Naturally, it is assumed that training and testing domains have some invariant properties or mechanisms, which allow generalization from one to another.
At a high level, all domain generalization approaches seek to capture those invariances, but do that differently.
A few of the possible directions of achieving domain generalization are: learning domain-invariant representations~\cite{muandet2013domain, ganin2016domain}, learning class-conditioned domain-invariant representations~\cite{li2018deep, galstyan2020robust, zhao2020domain}, using robust loss functions~\cite{Sagawa*2020Distributionally}, learning invariant causal predictors ~\cite{arjovsky2019invariant}, and using meta-learning~\cite{li2018learning}.
Most of the methods listed above outperform the straightforward empirical risk minimization (ERM) approach on toy domain generalization instances (e.g., colored MNIST).
However, Gulrajani and Lopez-Paz~\cite{domainbed} demonstrate that when evaluated on realistic DG instances, these methods are unable to outperform ERM significantly.
To improve domain generalization methods or propose new ones, we need to understand why and how do domain generalization methods fail.
This is the main goal of this paper.

Our contributions are threefold. First, we characterize the general failure modes of domain generalization methods: training set underfitting, test set inseparability, training-test misalignment and classifier non-invariance. We develop tools that measure the contribution of each of these failures in the total error of a given model.
Inspired by the popularity of the methods based on invariant representation learning, we also characterize failure modes related to achieving domain invariance.
Second, we identify two common patterns of generalization failures. In the first pattern, many algorithms achieve domain-invariant representations across the training domains, but not on unseen domains, while the generalization error is negatively correlated with the representation invariance on unseen domains. The second pattern is when domain invariance is increased across all domains, but the increase coincides with a degradation of the representations of unseen domains, thus limiting the accuracy of the models. 
Third, we show that by fixing the representations it is possible to isolate the classifier non-invariance failure, and significantly improve the generalization even with the most basic algorithms. These findings additionally confirm that domain-invariant representations are neither necessary nor sufficient for successful domain generalization.

The code for measuring contributions of failure modes is available at \url{github.com/YerevaNN/dom-gen-failure-modes}, while the code for reproducing the experiments is available at \url{github.com/TigranGalstyan/wilds}.


%% file: sections/related.tex
\section{Related Work}
The ability to generalize beyond the training distribution is the key goal of machine learning.
Torralba and Efros~\cite{torralba2011unbiased} show that common image classification datasets have significant differences, because of which methods trained on one dataset often fail to generalize well to other datasets.
The fact that learning on a single domain is susceptible to dataset biases and spurious correlations has been confirmed in many contexts~\cite{torralba2011unbiased, beery2018recognition,dai2018dark,albadawy2018deep,pmlr-v97-recht19a,geirhos2020shortcut}.

A few settings focus on generalization outside the training distribution.
The out-of-distribution (OOD) and robustness literature focus on the case when there is a distribution shift at test time~\cite{storkey2009training, quionero2009dataset}, including but not limited to label shifts~\cite{saerens2002adjusting, label-shift-Lipton}, co-variate shifts~\cite{shimodaira2000improving, gretton2009covariate}, conditional shifts~\cite{zhang2013domain}, visual distortions~\cite{dodge2017study, hendrycks2018benchmarking}, stylistic and other changes~\cite{hendrycks2020many}.
More general settings of domain generalization and domain adaptation assume that examples from multiple domains are available for training, with the difference that in domain adaptation a collection of examples (labeled or not) from the testing domain are available for adaptation~\cite{WANG2018135}.
In this paper we focus on the domain generalization problem (also called zero-shot domain adaptation~\cite{peng2018zero}), because of its generality and better correspondence with practical settings.
Nevertheless, most of the proposed techniques and definitions can be easily extended to domain adaptation.

A group of approaches aim to get domain generalization by learning domain-invariant representations~\cite{muandet2013domain, ganin2016domain, li2018deep, galstyan2020robust, zhao2020domain}.
DANN~\cite{ganin2016domain} uses an adversarial classifier to predict domain from representations, while C-DANN~\cite{li2018deep} learns such a classifier for each domain separately.
Galstyan et al.~\cite{galstyan2020robust} regularize empirical risk minimization (ERM) with a term that uses Hilbert-Schmidt independence criterion (HSIC)~\cite{gretton2005measuring} to make representations be independent from domains conditioned on labels.
Zhao et al.~\cite{zhao2020domain} use a variety of techniques to enforce the distribution of labels conditioned on representations be the same for all domains.
DeepCORAL~\cite{sun2016deep} adds a regularization term to align the second-order statistics of representations of different domains.
Invariant risk minimization (IRM)~\cite{arjovsky2019invariant} aims to learn invariant causal predictors, by finding a representation of data such that optimal classifiers on top of representations are the same for each domain.
Li et al.~\cite{li2018deep} use meta-learning with the one-step look-ahead gradient update technique to simulate evaluating on unseen domains during the training.
GroupDRO~\cite{Sagawa*2020Distributionally} minimizes the worst-case risk across training domains. 
Recently, methods based on gradient matching have been proposed~\cite{fish,fishr}.
Finally, a few works introduce  benchmarks for evaluating domain generalization methods~\cite{domainbed, koh2020wilds}.

Gulrajani and Lopez-Paz~\cite{domainbed} demonstrate that none of the existing domain generalization approaches outperform empirical risk minimization when evaluated on realistic tasks.
There is a very limited amount of research done on why and how these domain generalization methods fail.
Rosenfeld et al.~\cite{rosenfeld2021the} study the failure modes of the IRM in theoretical settings.
Nagarajan et al.~\cite{nagarajan2021understanding} explain the mechanisms by which ERM fails on very easy out-of-domain generalization tasks.
In contrast to these works, our analysis and proposed techniques below 
are applicable to any domain generalization algorithm.

%% file: sections/problem-setting.tex
\begin{figure*}
\centering
\begin{subfigure}{0.1725\linewidth}
\includegraphics[width=\linewidth]{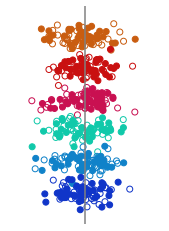}
    \caption{$\EGzero' \gg 0$}
    \label{fig:G0-fail}
\end{subfigure}
\begin{subfigure}{0.1725\linewidth}
\includegraphics[width=\linewidth]{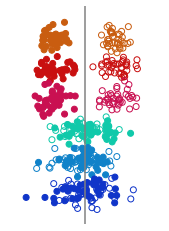}
    \caption{$\EGzero' = 0$, $\EGone \gg 0$}
    \label{fig:G1-fail}
\end{subfigure}
\begin{subfigure}{0.1725\linewidth}
\includegraphics[width=\linewidth]{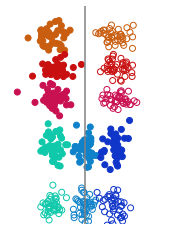}
    \caption{$\EGone' = 0$, $\EGtwo \gg 0$}
    \label{fig:G2-fail}
\end{subfigure}
\begin{subfigure}{0.1725\linewidth}
\includegraphics[width=\linewidth]{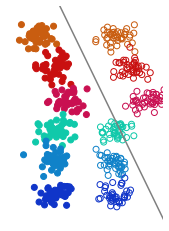}
    \caption{$\EGtwo' = 0$, $\EGthree \gg 0$}
    \label{fig:G3-fail}
\end{subfigure}
\begin{subfigure}{0.1725\linewidth}
\includegraphics[width=\linewidth]{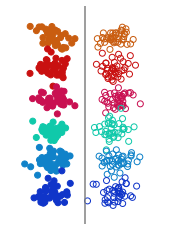}
    \caption{$\EGthree'=0$}
    \label{fig:G3-success}
\end{subfigure}
\caption{Failure modes of domain generalization algorithms demonstrated by 2D representations of a two-class data. Failure modes are described in \cref{sec:failuremodes}. Empty and full circles correspond to $0$ and $1$ classes. Colors encode the domains, training domains use colors close to red, test domains use colors close to blue. The lines correspond to decision boundaries likely to be found by a simple ERM algorithm trained on the training domains. In all subfigures the domains are distinguishable: $\DIzero \gg 0$.}
\label{fig:failures}
\end{figure*}

\section{Problem Setting and Notation}
Consider an input space $\mathcal{X}$ and output space $\mathcal{Y}$. 
A joint probability distribution $p(x,y)$ on $\mathcal{X} \times \mathcal{Y}$ is called a domain and defines a prediction task.
In the domain generalization task we assume there is a family $\mathscr{D}$ of domains that are somehow related to each other and correspond to similar prediction tasks.
The learner observes training data from $n_1$ domains, $p^1_1(x, y), \ldots, p^1_{n_1}(x, y)$.
The goal is to learn a predictor $f : \mathcal{X} \rightarrow \mathcal{Y}$ that generalizes to unseen domains from $\mathscr{D}$.
Note that in contrast to the domain \emph{adaptation} problem, here the learner cannot do any adaptation at inference time.

In this paper we focus on classification tasks, where $\mathcal{X} = \mathbb{R}^p$ and $\mathcal{Y} = \{1, 2, \ldots, C\}$.
We assume that, in addition to $n_1$ training domains, we have $n_2$ validation domains $p^2_1(x,y), \ldots, p^2_{n_2}(x,y)$, and $n_3$ test domains, $p^3_{1}(x,y), \ldots, p^3_{n_3}(x,y)$.
We assume that there is no label shift across domains: $p^1_1(y)=\ldots=p^1_{n_1}(y)=p^2_1(y)=\ldots=p^2_{n_2}(y)=p^3_1(y)=\ldots=p^3_{n_3}(y)$. We discuss this limitation in \cref{sec:app-label-shift}.
We also assume that for each domain $i=1,\ldots,n_j$, $j=1,2,3$, we are given a collection of independent samples $\mathcal{D}^j_i$ from the corresponding distribution $p^j_i(x,y)$. Each of these sets is further divided into two parts: $\mathcal{D}^j_i = \mathcal{T}^j_i \cup \mathcal{V}^j_i$. For simplicity, we also define the union of the samples of training (validation, testing) domains: $\mathcal{D}^j = \bigcup_{i=1}^{n_j} \mathcal{D}^j_i$, $\mathcal{T}^j = \bigcup_{i=1}^{n_j} \mathcal{T}^j_i$, $\mathcal{V}^j = \bigcup_{i=1}^{n_j} \mathcal{V}^j_i$, for each $j=1,2,3$.
The algorithms will be trained on samples from $\mathcal{T}^1$. The set $\mathcal{V}^1$ is called \textit{in-domain validation set} in \cite{koh2020wilds}, \textit{quasi-development set} in \cite{galstyan2020robust} and \textit{training-domain validation set} in \cite{domainbed}. It is used to measure the performance of the algorithms on unseen samples from the training domains. The performance on unseen domains is measured using $\mathcal{D}^3$. The sets $\mathcal{T}^j$ and $\mathcal{V}^j$, $j=2,3$, are used for analysis.


We consider domain generalization methods that use neural networks with two components: a feature extractor $z = h_\theta(x) \in \mathbb{R}^d$ with parameters $\theta \in \Theta$, and a classification head $\widehat{y} = f_w(z) \in \mathbb{R}^C$ with parameters $w \in \mathcal{W}$.
Throughout the paper we call $z$ the \emph{representation} of $x$.
Let $\ell : \mathbb{R}^C \times \mathcal{Y} \rightarrow \mathbb{R}$ be a loss function that measures discrepancy between a prediction and a label.
In this work $\ell$ is chosen to be the standard 0-1 loss function: $\ell(\widehat{y}, y) = \mathbf{1}\{\text{argmax}_{k} \widehat{y}_k(z(x)) \neq y\}$, which is the most popular evaluation metric in classification tasks.
Nevertheless, most of the results of this paper can be easily extended to other choices of loss functions, such as the negative log-likelihood loss function, often used for \emph{training} classifiers.






%% file: sections/failure-modes.tex
\section{Failure Modes}\label{sec:failuremodes}

We propose simple evaluation metrics to diagnose a trained model and identify a set of failures that contribute to the final error on unseen domains.
We present simplified schematic visualizations of 2-dimensional representation spaces corresponding to each of the failure modes, \eg \cref{fig:failures}. In all such figures each circle corresponds to a representation of one sample. The domain of a sample is encoded by the color of its circle. Orange-red-pink colors correspond to the training domains, while green-blue colors correspond to the test domains. Filled and empty circles are used to encode the binary labels.

To formally define the generalization and invariance metrics, we need the following additional notation.
Let $(X^1_1, Y^1_1), \ldots (X^1_{n_1},Y^1_{n_1})$ be random variables drawn from training domains $p^1_1(x,y), \ldots, p^1_{n_1}(x,y)$ and $(X^3_1, Y^3_1), \ldots (X^3_{n_3},Y^3_{n_3})$ be random variables drawn from test domains $p^3_1(x,y), \ldots, p^3_{n_3}(x,y)$.
Let $(X^{1,3},Y^{1,3})$ be a random variable drawn from the mixture of all training and test domains $p^{1,3}(x,y)=\frac{1}{n_1 + n_3}\rbr{\sum_{i=1}^{n_1}p^1_i(x,y) + \sum_{i=1}^{n_3}p^3_i(x,y)}$.

\begin{figure*}
\centering
\begin{subfigure}{0.1725\linewidth}
    \includegraphics[width=\linewidth]{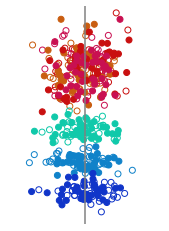}
    \caption{$\EGzero' \gg 0$}
    \label{fig:I1-G0-fail}
\end{subfigure}%
\begin{subfigure}{0.1725\linewidth}
    \includegraphics[width=\linewidth]{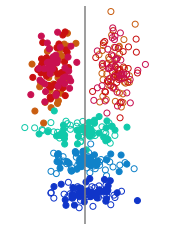}
    \caption{$\EGzero' = 0$, $\EGone' \gg 0$}
    \label{fig:I1-G1-fail}
\end{subfigure}%
\begin{subfigure}{0.1725\linewidth}
    \includegraphics[width=\linewidth]{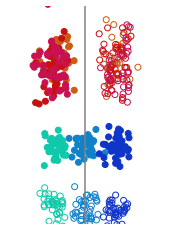}
    \caption{$\EGone' = 0$, $\EGtwo' \gg 0$}
    \label{fig:I1-G2-fail}
\end{subfigure}%
\begin{subfigure}{0.1725\linewidth}
    \includegraphics[width=\linewidth]{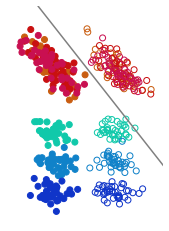}
    \caption{$\EGtwo' = 0$, $\EGthree' \gg 0$}
    \label{fig:I1-G3-fail}
\end{subfigure}%
\begin{subfigure}{0.1725\linewidth}
    \includegraphics[width=\linewidth]{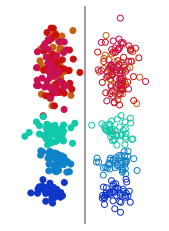}
    \caption{$\EGthree' = 0$}
    \label{fig:I1-fail-G3-success}
\end{subfigure}

\begin{subfigure}{0.1825\linewidth}
    \includegraphics[width=\linewidth]{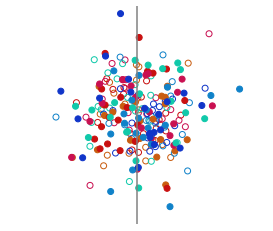}
    \caption{$\EGzero' \gg 0$}
    \label{fig:I2-G0-fail}
\end{subfigure}
\begin{subfigure}{0.1825\linewidth}
    \includegraphics[width=\linewidth]{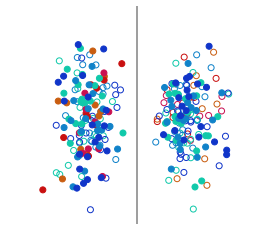}
    \caption{$\EGzero' = 0$, $\EGone' \gg 0$}
    \label{fig:I2-G1-fail}
\end{subfigure}
\begin{subfigure}{0.1825\linewidth}
    \includegraphics[width=\linewidth]{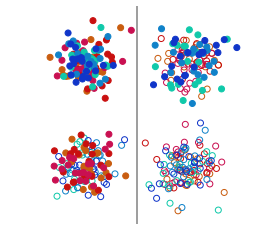}
    \caption{$\EGone' = 0$, $\EGtwo' \gg 0$}
    \label{fig:I2-G2-fail}
\end{subfigure}
\begin{subfigure}{0.1825\linewidth}
    \includegraphics[width=\linewidth]{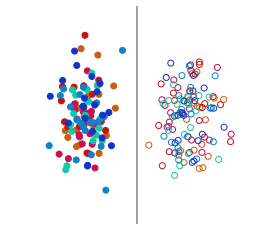}
    \caption{$\EGthree' = 0$, $\DItwo' = 0$}
    \label{fig:I2-success-G3-success}
\end{subfigure}

\caption{Failure modes of domain generalization algorithms demonstrated by 2D representations of a two-class data. The first row corresponds to training domain invariance ($\DIzero' = 0$), but  $\DIone' \gg 0$. 
The second row corresponds to training-test domain invariance ($\DIone' = 0$).
In particular, the first three images demonstrate that $\DIone'=0$ can co-occur with large  $\EGzero$, $\EGone$ and  $\EGtwo$. The right-most image demonstrates that $\EGzero'=0$ and $\DItwo'=0$ together imply $\EGthree'=0$.
Failure modes are described in \cref{sec:failuremodes}.}
\label{fig:invariance-failures_i1}
\end{figure*}

\subsection{Generalization metrics}
Below we define four evaluation metrics that capture qualitatively different aspects of generalization.
All of this metrics will be some kind of errors (so lower better). Hence, we will use terms ``metric'' and ``error'' interchangeably.

\textbf{Training set underfitting.} 
How well does the model perform on training domains? Formally, this metric is denoted by $\EGzero'$ and is defined as follows:
\begin{equation*}
\EGzero' \triangleq \frac{1}{n_1} \sum_{i=1}^{n_1} \E_{X^1_i, Y^1_i}\sbr{ \ell\rbr{f_w(h_\theta(X^1_i)), Y^1_i}},
\end{equation*}
where $h_\theta$ and $f_w$ are the learned feature extractor and  classifier, respectively.
This metric is expected to be small for most domain generalization algorithms.
A model might have large $\EGzero'$ if the regularization terms of its objective were so strong compared to the classifier loss that the feature extractor failed to learn anything useful. 
An example will be discussed in the Experiments section.
Training set underfitting is demonstrated in \cref{fig:G0-fail}.

\textbf{Test set inseparability.}
How well are the representations of the test domains separable with respect to the chosen class of classifier heads?
Formally, we define test set inseparability error as follows:
\begin{align*}
\EGone' &\triangleq\inf_{w' \in \mathcal{W}}\cbr{\frac{1}{n_3}\sum_{i=1}^{n_3} \E_{X^3_i, Y^3_i}\sbr{ \ell\rbr{f_{w'}(h_\theta(X^3_i)), Y^3_i} }}.
\end{align*}
This metric will be large for models whose feature extractor is overfitted on the training domains and does not produce a reasonable representation for the test domains.
Note that the quality of the representation is measured with respect to the class of classifier heads (e.g., linear functions), because we are not interested in cases when the representations have information about domains but that information is not decodable/usable by the considered family of classifiers.
The case when $\EGzero'=0$ and $\EGone'$ is large is demonstrated in \cref{fig:G1-fail}.

\textbf{Training-test misalignment.} 
Is there a common classifier for the representations of training and test domains?
The extent of the answer to this question being negative is measured by the $\EGtwo'$ metric:

\begin{align*}
    \EGtwo' &\triangleq \frac{1}{n_3}\sum_{i=1}^{n_3} \E_{X^{3}_i,Y^{3}_i}\sbr{\ell(f_{\tilde{w}}(h_\theta(X^{3}_i)), Y^{3}_i)},\\
    \text{ where } \tilde{w} &\in\arg\min_{\tilde{w} \in \mathcal{W}} \E_{X^{1,3},Y^{1,3}}\sbr{\ell(f_{\tilde{w}}(h_\theta(X^{1,3})), Y^{1,3})}.
\end{align*}

If $\EGone'$ is small, the representations of the samples from each domain are good enough to be separated by a domain-specific classifier. 
But it might be impossible to find a classifier that would separate samples from all domains (both training and test) at once, resulting in relatively high $\EGtwo'$ error.
This case is demonstrated in \cref{fig:G2-fail}.
Such scenarios can be thought of as ``milder'' versions of overfitting, as the feature extractor learned useful and decodable information, but was not able to distribute the representations in a consistent way across domains.

\textbf{Classifier non-invariance.} This final generalization metric is the standard test error and measures the performance of the learned model on test domains.
Formally, the metric $\EGthree'$ is defined as follows:
\begin{equation*}
\EGthree' \triangleq \frac{1}{n_3}\sum_{i=1}^{n_3}\E_{X^3_i,Y^3_i}\sbr{\ell(f_w(h_\theta(X^3_i)),Y^3_i)}.
\end{equation*}
If $\EGtwo' = 0$, the representations are so good that there exists a classifier that can separate samples from both training and test domains with significant success.
In such cases $\EGthree'$ essentially measures whether the training algorithm was able to find a classifier that works for both training and test domains.
Hence, the name of this metric $\EGthree'$ is ``classifier non-invariance''.
Note that a training algorithm might easily fail to select an invariant classifier, as it does not have access to the test data during training. This scenario is schematically demonstrated in \cref{fig:G3-fail}.

Note that the four failure modes defined above and depicted in \cref{fig:failures} are extreme cases and should not be thought as a comprehensive list of cases.
In \cref{sec:decompositions} we propose a way of attributing the model's overall test error to each of these failure modes.

\subsection{Invariance metrics}
Many domain generalization algorithms attempt to learn domain-invariant representations. 
The intuition is that if the representations of the samples are similar across all domains, then a classifier designed for the training domains will generalize to the test domains.
As the algorithms have access only to the training data, they usually achieve domain invariance across the training domains, but not across the union of training and test domains.

Motivated by the generalization metrics defined above, we introduce metrics that assess qualitatively different aspects of domain invariance of learned representation.
While there are many ways to measure the extent of invariance of representations across two or more domains (e.g., using formal distances or divergences between probability distributions), we choose to use a technique similar to the one used in generalization metrics above.
Informally, we will measure how invariant are the representations of samples of two domains by measuring how well one can differentiate them using a domain classifier.
Importantly, we will draw domain classifier from the same family of functions as label classifiers.
That is, if the label classifier head uses a specific architecture, we would use the same architecture (with a different number of outputs) for domain classifiers.
This intentionally ignores domain information in representations that cannot be decoded by the label classifier.
Domain classifiers will be functions of the form $g_\omega: \mathbb{R}^d \rightarrow \mathbb{R}^K$, $\omega \in \Omega$, 
where $K$ is the number of domains.

\textbf{Training domain distinguishability.}
Are the representations of examples from the training domains domain-invariant?
Formally, we denote our first invariance metric $\DIzero'$ and define it the following way:
\begin{align*}
\small
\DIzero' &\triangleq
1 - \inf_{\omega \in \Omega} \sbr{\frac{1}{n_1} \sum_{i=1}^{n_1}{ \E_{X^1_i}\sbr{\ell(g_{\omega}(h_\theta(X^1_i)), i)}}} - \frac{1}{n_1}.
\end{align*}
Here the constant $\frac{1}{n}$ is the accuracy of the random classifier.
Note that lower values of $\DIzero'$ correspond to higher invariance, and $\DIzero'=0$ implies it is impossible to distinguish domains better than the trivial baseline.
Large values of $\DIzero'$ can arise when the regularizer designed to induce invariance is too weak compared to the classification loss. 
In practice, most algorithms do not end up in this state if the regularization term is tuned correctly.
It is important to note that achieving domain invariance on the training domains is not necessary to have domain generalization, as demonstrated in \cref{fig:G3-success}.

\textbf{Training-test domain distinguishability.}
Are the representations of examples from the union of the training and test domains domain-invariant?
Our second invariance metric $\DIone'$ measures the extent of this question having a positive answer:
\begin{align*}
\small
\DIone' &\triangleq 1 - \inf_{\omega \in \Omega}\E_{X^1_{1:n_1},X^3_{1:n_3}}\left[\frac{1}{n_1 + n_3}\left(\sum_{i=1}^{n_1} \ell(g_{\omega}(h_\theta(X^1_i)), i) \right.\right.\\
&\quad+\left.\left.\sum_{i=1}^{n_3}\ell(g_{\omega}(h_\theta(X^3_i)), i + n_1)\right)\right] - \frac{1}{n_1 + n_3}.
\end{align*}
Assuming that $\DIzero = 0'$, this metric $\DIone'$ can be large if the distributions of the representations of the training and test sets do not coincide; if there is no domain invariance across the test domains; or both.
In theory, this failure of reaching invariance can coincide with all failure modes of generalization, as shown in \cref{fig:invariance-failures_i1}.
Moreover, it is possible to have domain generalization along with large $\DIone'$ (\cref{fig:I1-fail-G3-success}).




Even when a model achieves high training-test domain invariance (low $\DIone'$), it is still possible that representations of Class 1 samples of the first domain coincide with the representations of Class 2 samples of the second domain and vice versa. 
In this case, there will be no domain generalization.
The next metric is designed to capture such situations.

\textbf{Training-test class-conditional domain invariance.} Are the representations of samples belonging to each of the classes domain-invariant?
Let $E_y$ denote the event $(Y^1_1=y \wedge \cdots \wedge Y^1_{n_1}=y \wedge Y^3_1=y \wedge \cdots \wedge Y^3_{n_3}=y)$.
We define the $\DItwo'$ metric as follows:
\begin{align*}
\small
\DItwo' &\triangleq 1- \frac{1}{C}\sum_{y=1}^C \inf_{\omega \in \Omega}\E\left[\frac{1}{n_1 + n_3}\left(\sum_{i=1}^{n_1} \ell(g_{\omega}(h_\theta(X^1_i)), i) \right.\right.\\ &\quad+\left.\left.\sum_{i=1}^{n_3}\ell(g_{\omega}(h_\theta(X^3_i)), i + n_1)\right) \bigg\lvert E_y\right] -  \frac{1}{n_1 + n_3}.
\end{align*}
This $\DItwo'$ metric can be seen as the conditional version of the previous metric $\DIone'$.

\subsection{Relations between domain invariance and generalization metrics}\label{sec:propositions}
We prove two propositions that establish connections between generalization and invariance failures.
In particular, these propositions rule out some combinations of invariance and generalization failures.
We first formally define domain invariance of representations and class-conditional domain invariance of representations (not to be confused with the corresponding invariance metrics).

\begin{definition}[Domain invariance of representations]
Let $\mathscr{D}$ be a family of domains. We say that the representations learned by a feature extractor $z=h(x)$ are domain-invariant with respect to $\mathscr{D}$ if for any two domains $p_1(x,y)$ and $p_2(x,y)$ from $\mathscr{D}$, the distributions of representations are equal, i.e., $\forall z, p_{Z_1}(z) = p_{Z_2}(z)$, where $X_1 \sim p_1(x), X_2 \sim p_2(x), Z_1 = h(X_1)$ and $Z_2=h(X_2)$.
\end{definition}
Likewise, we define invariance of representations conditioned on labels by requiring $p(z|y)$ to be the same for across all domains of the family $\mathscr{D}$.

\begin{proposition}
If $\EGtwo' = 0$, then domain invariance of representations w.r.t. the union of training and testing domains implies class-conditioned domain invariance w.r.t. the union of training and testing domains.
\label{prop:one}
\end{proposition}

\begin{proposition}
If $\EGzero'=0$ and representations are class-conditioned domain-invariant w.r.t. the union of training and testing domains, then $\EGthree' = 0$.
\label{prop:two}
\end{proposition}
The second proposition implies that for a class-conditioned domain-invariant model, perfect performance on the training domains is sufficient for domain generalization.

\subsection{Decomposition of errors}\label{sec:decompositions}
The generalization metrics defined above are sequential in the sense that if $e'_i$ is large then $e'_{i+1}$ is likely to be large as well.
For this reason, the differences $e'_{i+1}-e'_{i},\ i=0,1,2$, are more suitable quantities for analysis purposes.
In fact, the failure modes depicted in \cref{fig:failures} are cases when one of these differences is large in conjunction with the previous metric being close to zero.
Following this reasoning, we define $\EGzero \triangleq \EGzero'$, $\EGone \triangleq \EGone' - \EGzero'$, $\EGtwo \triangleq \EGtwo' - \EGone'$, $\EGthree \triangleq \EGthree' - \EGtwo'$, and decompose the error $e\triangleq \EGthree'$ on the test domains into four components:
\begin{align*}
    e = &\underbrace{\EGzero}_{\substack{\text{training set} \text{ underfitting}}} 
    + \underbrace{\EGone}_{\substack{\text{test set} \text{ inseparability}}}\\
    &\quad+ \underbrace{\EGtwo}_{\substack{\text{training-test} \text{ misalignment}}} 
    + \underbrace{\EGthree}_{\substack{\text{classifier} \text{ non-invariance}}}.\numberthis
\end{align*}
Each of these components can be interpreted as the individual contribution of the corresponding failure mode to the error of the model.
This decomposition is the most meaningful when the its components are nonnegative.
In general, $\EGone$, $\EGtwo$, and $\EGthree$ can be negative, for example when samples of testing domains are significantly easier to classify compared to those of training domains.
However, this is a rare phenomenon and has not been observed in our experiments.
Moreover, one can prove that on average (w.r.t. to the training-test splits of the domains) each $e_i,\ i=0,1,2,3$ is nonnegative.
We give the precise formulation of this statement along with its proof in \cref{sec:decomposition-statements}. 

Similarly, domain-distinguishability $d \triangleq \DItwo'$ can be decomposed into three components:
\begin{align*}
    d & = \underbrace{\DIzero}_{\substack{\text{training} \ \text{domain} \ \text{distinguishability}}} 
    + \underbrace{\DIone}_{\substack{\text{training-test} \ \text{domain} \ \text{distinguishability}}}\\
    &\quad+ \underbrace{\DItwo}_{\substack{\text{training-test} \ \text{class-conditional} \ \text{domain distinguishability}}},\numberthis
\end{align*}
where $\DIzero \triangleq \DIzero'$, $\DIone \triangleq \DIone' - \DIzero'$, and $\DItwo \triangleq \DItwo' - \DIone'$. Again, each of these components can be interpreted as the individual contribution of the corresponding failure mode.
Some components of this decomposition also can be negative in rare situations but are nonnegative if we consider averaging over training-test splits of the domains (see \cref{sec:decomposition-statements}).

%% file: sections/experiments.tex
\section{Experiments}\label{sec:experiments}
\subsection{Datasets and algorithms}
The number of datasets suitable for testing domain generalization algorithms is not large. An attempt to collect them under a unified format was done in \cite{domainbed}. Two of the seven datasets are based on MNIST digits, four others are simple unions of unrelated datasets with similar labels, and only one is realistic: TerraIncognita. This latter one has a serious label shift between domains. 
Recently proposed WILDS benchmark \cite{koh2020wilds} has another seven datasets which are connected to real-world problems. 
Only one of the seven, \textbf{Camelyon17}, is carefully designed to have no label shift. It is a patch-based variant of the larger Camelyon17 dataset \cite{camelyon} of lymph node captures. The version in the WILDS benchmark contains 450000 96x96 patches of images of cancer metastases in lymph node sections. The label for each patch is binary indicating whether the patch contains any tumor tissue. The domain of a patch is the hospital from where the image comes from. There are five different hospitals, three for training, one for validation and one for testing. Following \cite{koh2020wilds}, we use Densenet-121 \cite{densenet} for all algorithms on this dataset and train for 10 epochs.

\begin{figure*}
\centering
\begin{subfigure}{0.23\linewidth}
    \includegraphics[width=\linewidth]{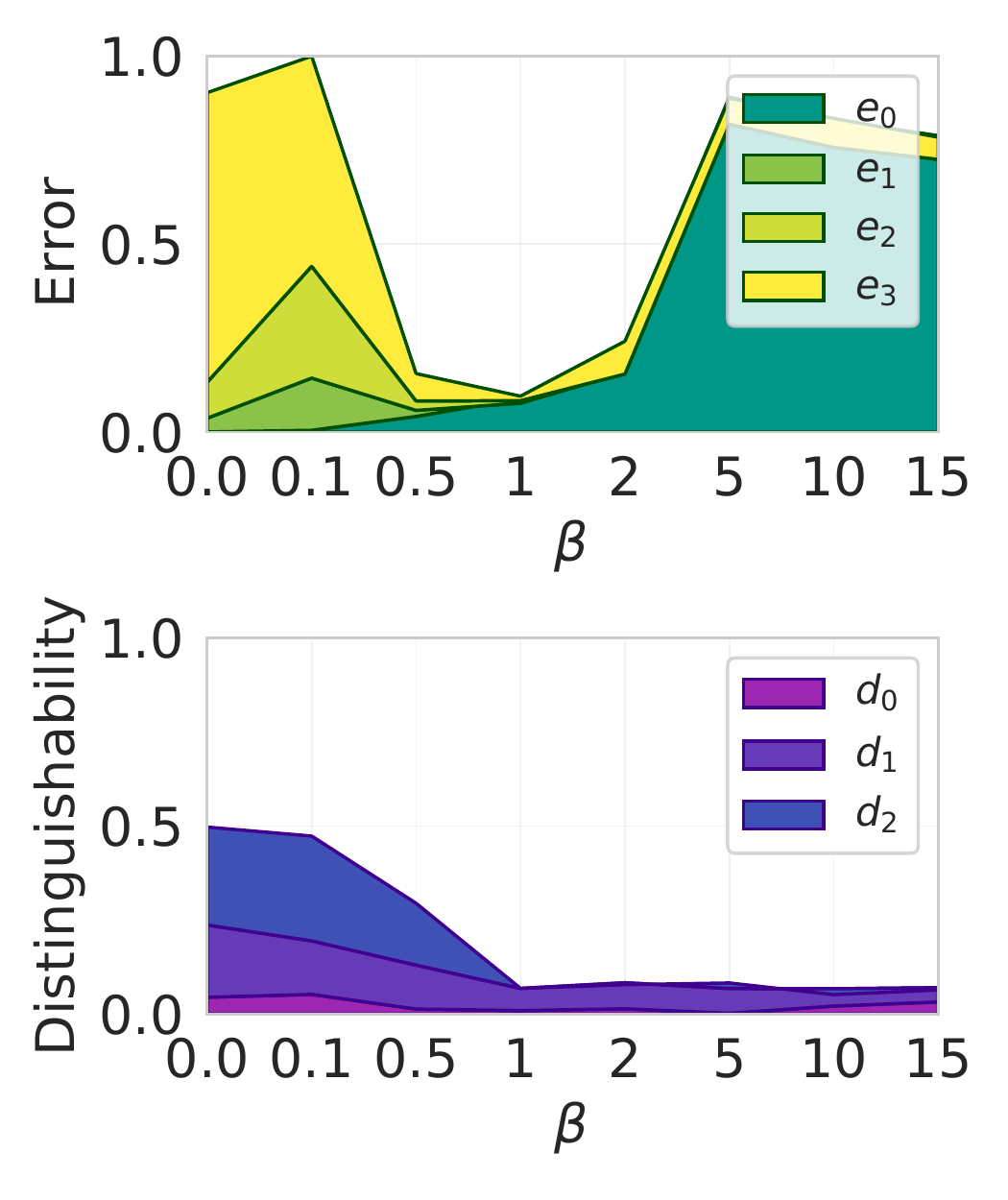}
    \caption{ERM+HSIC}
    \label{fig:CMNIST_HSIC}
\end{subfigure}
\begin{subfigure}{0.23\linewidth}
    \includegraphics[width=\linewidth]{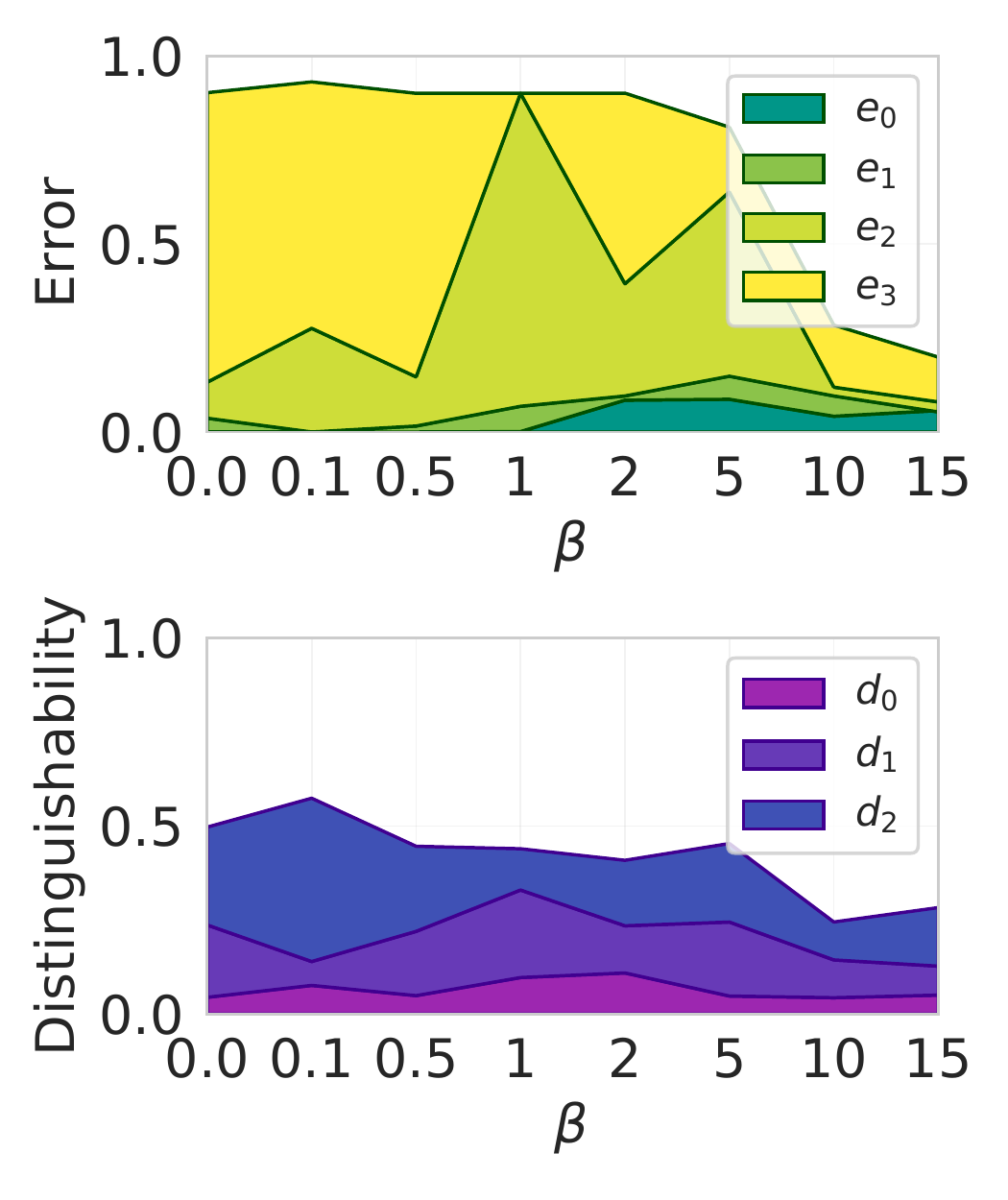}
    \caption{DeepCORAL}
    \label{fig:CMNIST_DeepCORAL}
\end{subfigure}
\begin{subfigure}{0.23\linewidth}
    \includegraphics[width=\linewidth]{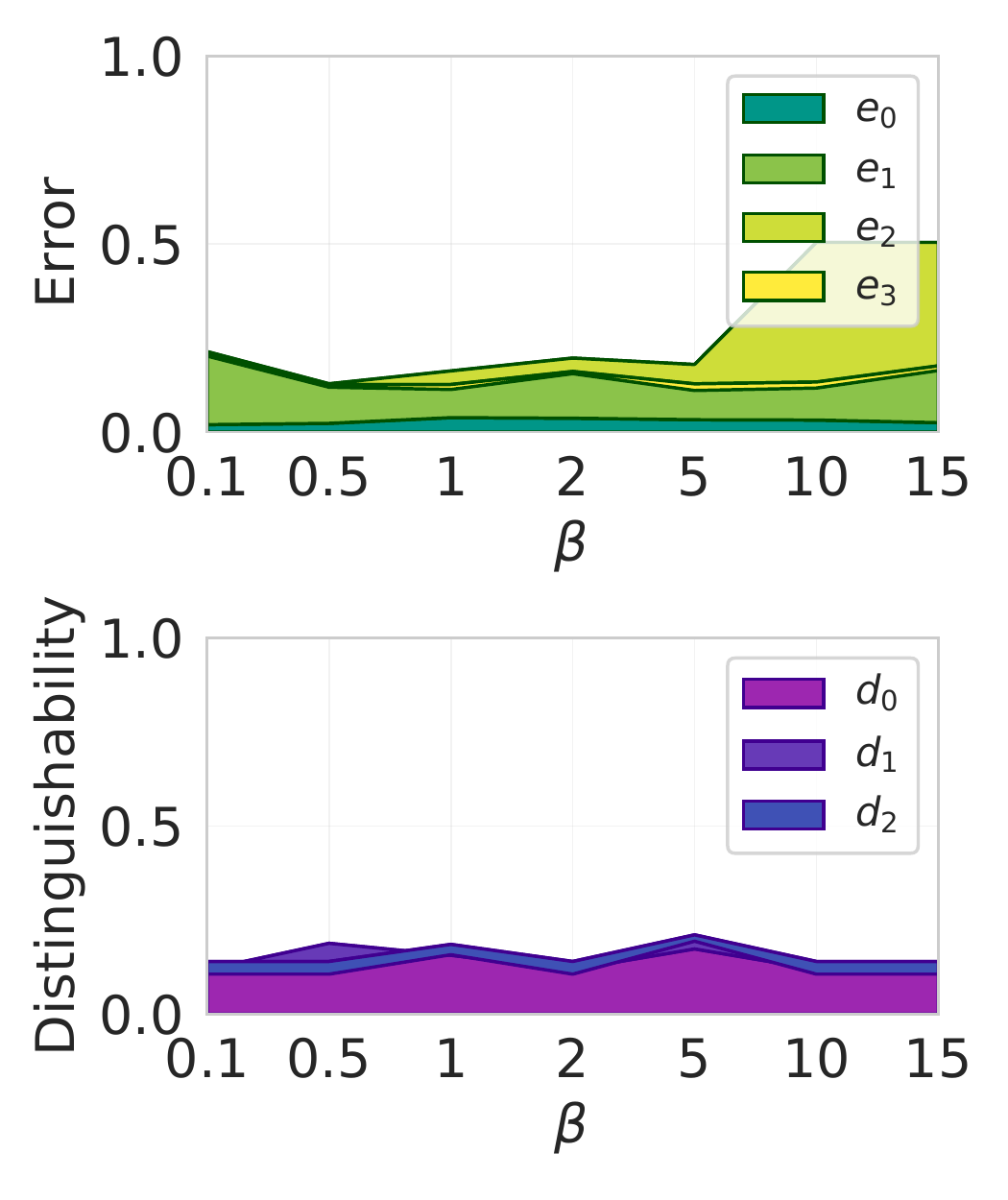}
    \caption{DANN}
    \label{fig:Camelyon_DANN}
\end{subfigure}
\begin{subfigure}{0.23\linewidth}
    \includegraphics[width=\linewidth]{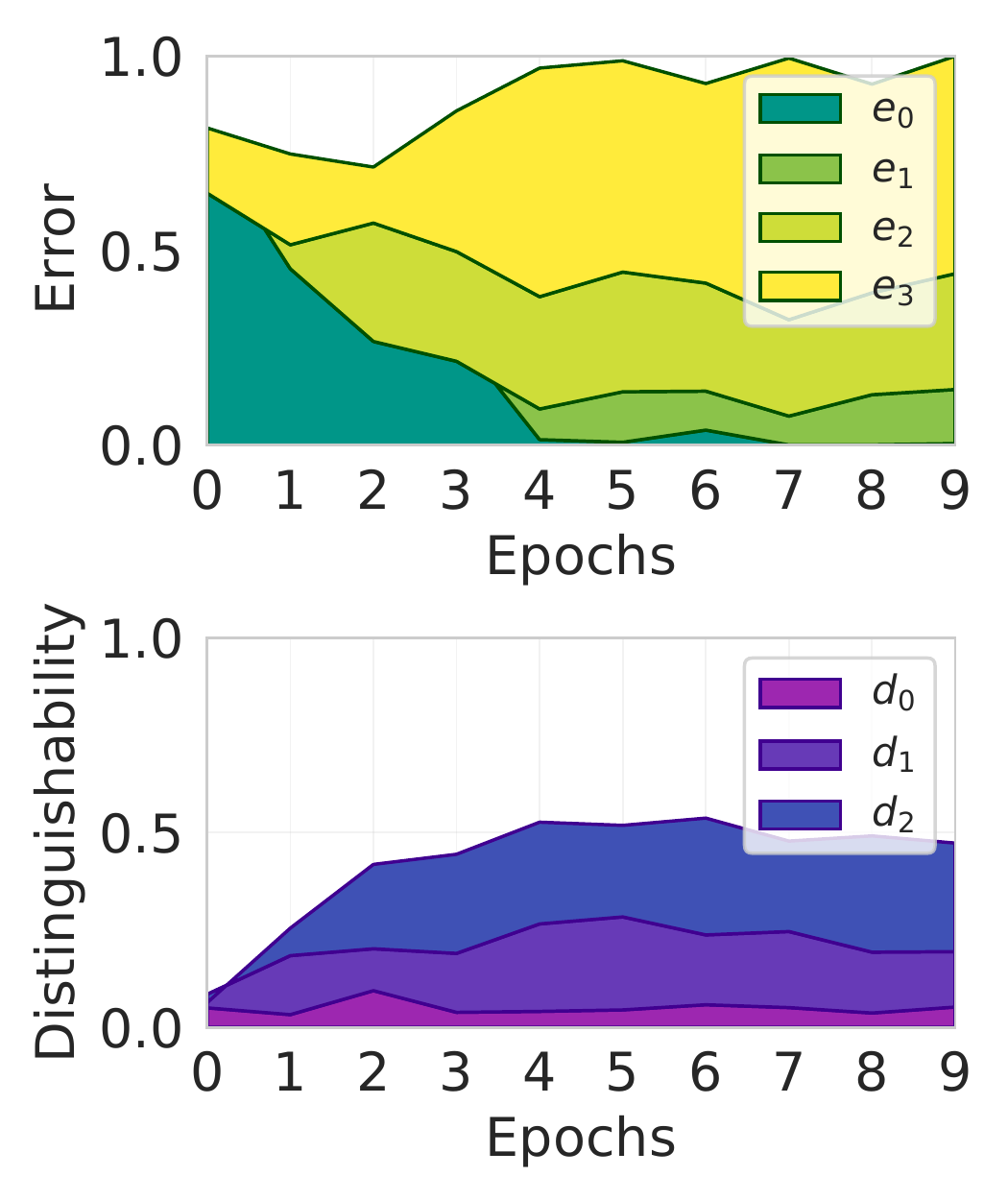}
    \caption{ERM+HSIC ($\beta=0.1$)}
    \label{fig:CMNIST_HSIC_beta2}
\end{subfigure}

\begin{subfigure}{0.23\linewidth}
    \includegraphics[width=\linewidth]{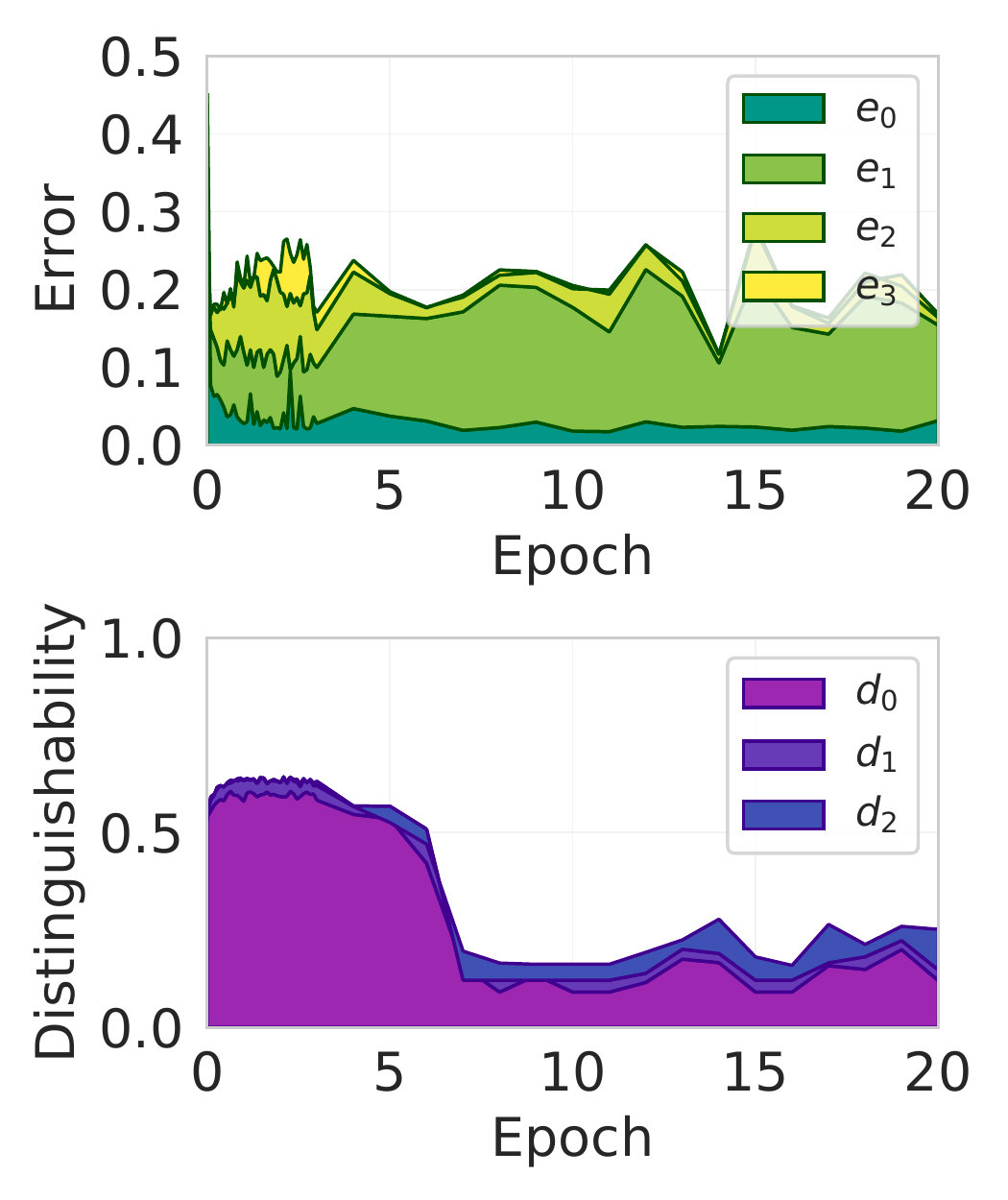}
    \caption{Random init (validation)}
    \label{fig:Camelyon17_ERM_random_val}
\end{subfigure}
\begin{subfigure}{0.23\linewidth}
    \includegraphics[width=\linewidth]{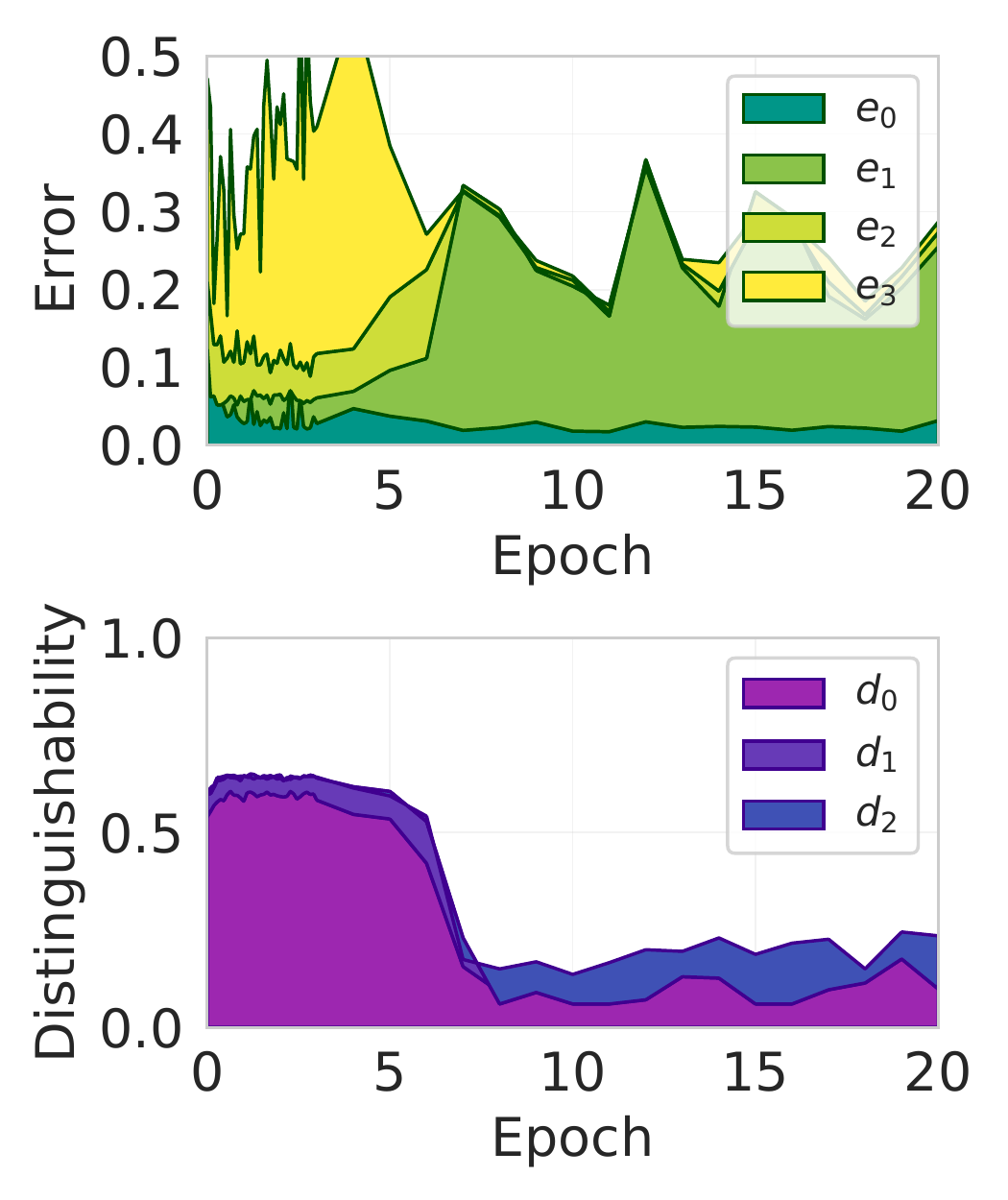}
    \caption{Random init (test)}
    \label{fig:Camelyon17_ERM_random_test}
\end{subfigure}
\begin{subfigure}{0.23\linewidth}
    \includegraphics[width=\linewidth]{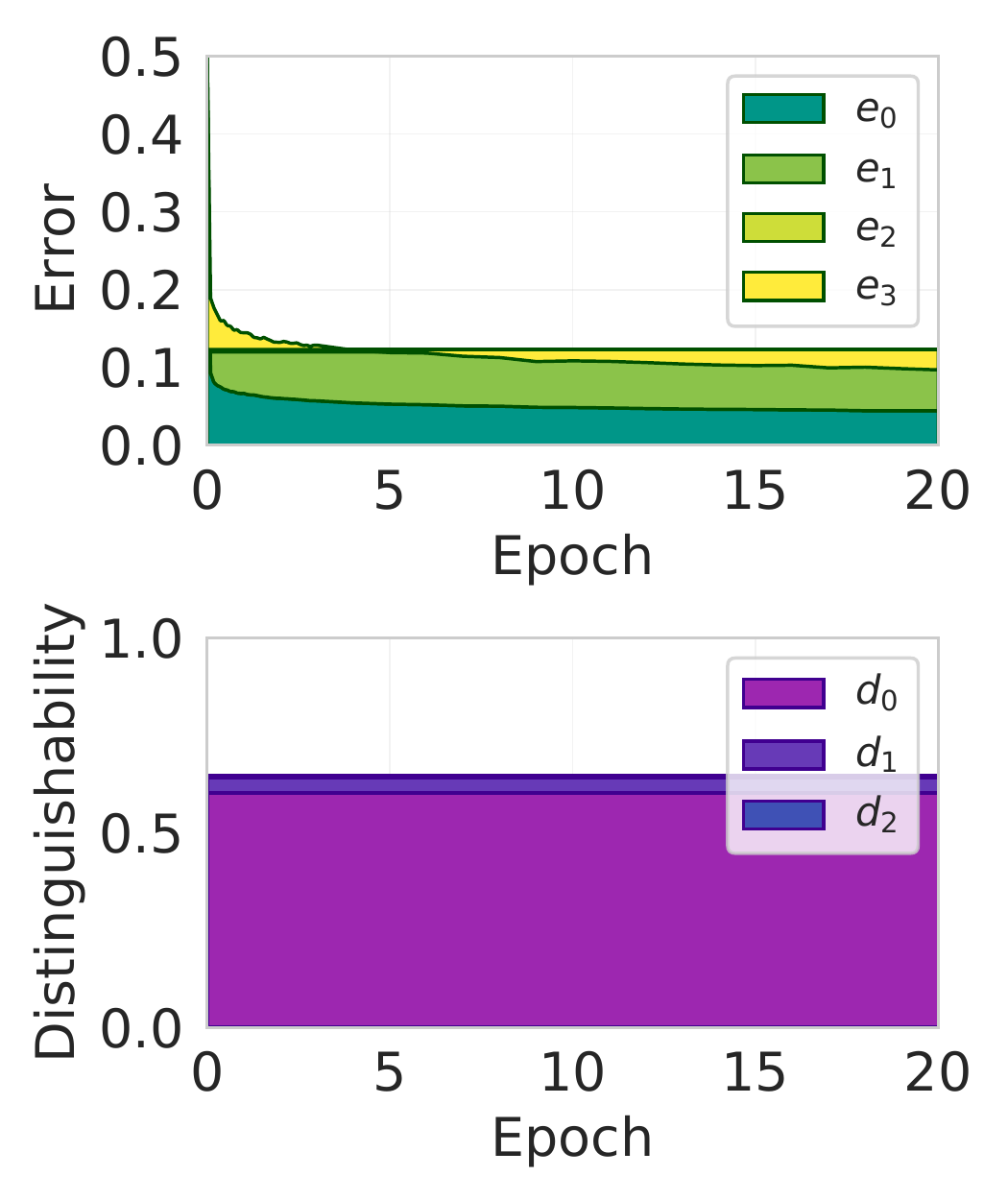}
    \caption{Frozen BYOL (validation)}
    \label{fig:Camelyon17_ERM_frozenBYOL_val}
\end{subfigure}
\begin{subfigure}{0.23\linewidth}
    \includegraphics[width=\linewidth]{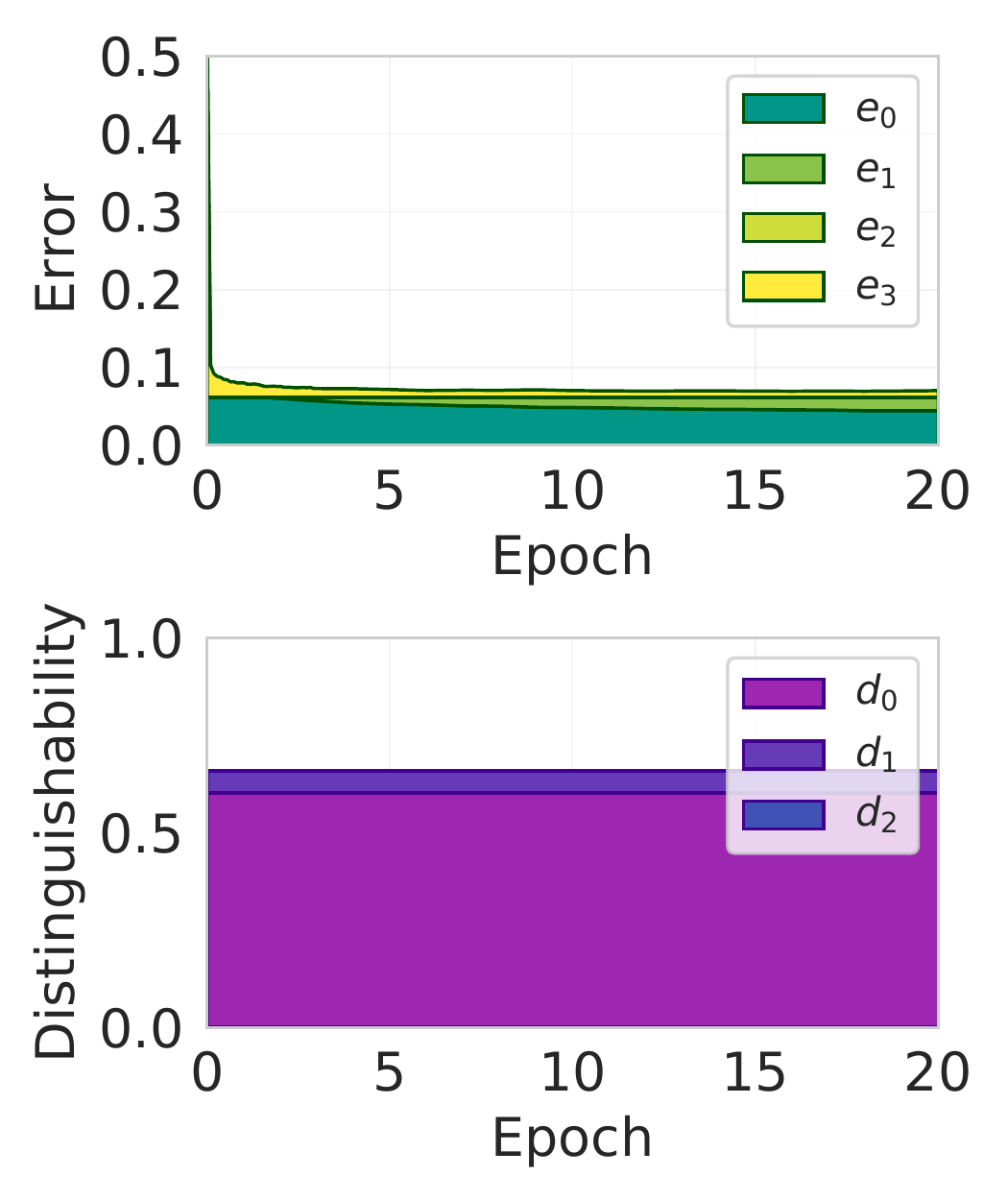}
    \caption{Frozen BYOL (test)}
    \label{fig:Camelyon17_ERM_frozenBYOL_test}
\end{subfigure}
\caption{Decomposition of generalization errors and domain-distinguishability of several models measured on the validation domains. Horizontal axis corresponds to: (a)-(f) regularization strength of the algorithms, (g) training epochs, (h) model complexity. Figures (e) and (f) show results on Camelyon17 dataset; all others are on Colored MNIST.}
\label{fig:results}
\end{figure*}

We also use a slightly modified version of \textbf{Colored MNIST} dataset from \cite{galstyan2020robust}, originally from \cite{lecun-mnist}. We changed the dataset generation process to mimic Camelyon17. Namely, we construct five domains by splitting MNIST images into five sets of equal size. For each domain we randomly fix three ``colors'', which are essentially 50 dimensional vectors, one for each digit. Then we ``colorize'' each image with the corresponding color. We end up with around 1200 images in each domain of shape 50x28x28. All models trained on this dataset use a simple neural network with a two-layer ReLU-activated convolutional feature extractor and a linear layer on top of it. All models are trained for 10 epochs. Samples from the datasets are presented in \cref{sec:app-dataset-samples}.


We analyze the following domain generalization algorithms. Empirical Risk Minimization (\textbf{ERM}) \cite{ermVapnik} is used as a baseline. It minimizes the sum of errors on all training domains and does not use domain information. \textbf{ERM+HSIC} \cite{galstyan2020robust} is a simple algorithm that attempts to induces domain invariance for training domains by adding a Hilbert-Schmidt Independence Criterion (HSIC)~\cite{gretton2009covariate} regularization term that penalizes domain information in the learned representations.
\textbf{DeepCORAL}~\cite{sun2016deep} was introduced as a domain adaptation algorithm, but was recently applied in domain generalization settings ~\cite{domainbed}. It penalizes the difference between the means and covariance matrices of representation distributions across domains. 
Invariant Risk Minimization (\textbf{IRM}) algorithm \cite{arjovsky2019invariant} tries to push representation distributions for all domains to have the same optimal classifier head. 
In \cite{gradient_starvation} the authors analyze the phenomenon of gradient starvation in algorithms based on gradient descent. They propose a new regularization method called Spectral Decoupling (\textbf{SD}). To overcome gradient starvation, it penalizes model's confidence by adding an L2 penalty on networks logits. \textbf{GroupDRO} was used by \cite{groupDRO} to tackle poor worst-group performance. It tries to increase worst-group performance by avoiding spurious correlations in the training data. By interpreting domains as groups, this algorithm becomes applicable to our setting. Another method relying on domain-invariant features is Domain-Adversarial Neural Network (\textbf{DANN}) introduced in \cite{dann}. It tries to achieve invariance by using a domain classifier (from features) and a gradient reversal layer. All of these algorithms (except ERM) have a regularization strength hyperparameter, which we denote by $\beta$. Hyperparameter ranges can be found in \cref{sec:hyperparams}.

Following literature, we consider the outputs of the last hidden layer as the learned representations $h_\theta(x)$.
Hence, label and domain classifiers $f_w(z)$ are linear (i.e., just a single fully-connected layer).

\textbf{Model selection.}
The authors of DomainBed benchmark explicitly warned against improper model selection methods in domain generalization.
They advocated that in practice model selection should be done either based on the performance on in-domain validation sets ($\mathcal{V}^1$), or by using leave-one-domain-out validation.
WILDS benchmark introduced separate validation domains and suggested to use them for model selection. In our experiments we will use this method for hyperparameter selection. To select the best checkpoint inside a single run, one can use the same validation domains. In this paper we always picked the last epoch.

\subsection{Measuring the failures}
\label{sec:tool}
To avoid accessing the test domains, we perform analysis on training and validation domains only.
The metrics defined in \cref{sec:failuremodes} contain expectations over distributions and infimums over classifiers. In practice, we approximate infimums by empirical risk minimization using the implementation of logistic regression in scikit-learn package~\cite{scikit-learn}. Note that this might fail to find the optimum, and the empirical estimates might be worse than the true values. In case of $\EGtwo'$, the estimate can be worse than the solution found by the domain generalization algorithm itself, i.e. $\EGthree' < \EGtwo'$. In fact, we encounter this phenomenon later in \cref{tab:frozenBYOL}. 
Following standard machine learning practices, we approximate distributions $p^1_i(x,y)$ and $p^2_j(x,y)$ with corresponding empirical distributions $\mathcal{T}^1_i$ and $\mathcal{T}^2_j$ during training, and with $\mathcal{V}^1_i$ and $\mathcal{V}^2_j$ when reporting the scores ($i=1,\ldots,n_1$, $j=1,\ldots,n_2$). 
In practice, to make $\DIone'$ and $\DItwo'$ comparable with $\DIzero'$, we train and evaluate domain classifiers on the union of $n_1-n_2$ training domains and $n_2$ validation domains.



%% file: sections/discussion.tex
\section{Results and Discussion}

Our analysis shows distinct patterns of failures on the two datasets. On \textbf{Colored MNIST}, all algorithms, including the ERM baseline ($\beta=0$), achieve domain invariance on training domains, but unseen domains remain quite distinguishable. At the same time, the representations of each domain remain separable ($e_0$ and $e_1$ are close to zero), but we have well-expressed training-test misalignment and/or classifier non-invariance across all algorithms and hyperparameters (\cref{fig:CMNIST_DeepCORAL,fig:CMNIST_SD,fig:CMNIST_GroupDRO}). The error on the validation set is usually larger than $0.8$. We observe a positive correlation ($0.4$-$0.5$) between $\DIone'$ and generalization error on the validation set ($\EGthree'$) (\cref{fig:corr_e3d1_DeepCORAL_CMNIST}).

The only exception to this pattern is seen on ERM+HSIC algorithm (\cref{fig:CMNIST_HSIC}). When the regularization strength $\beta$ approaches $1$, we achieve full invariance across all domains, training-test misalignment approaches zero and the classifier becomes invariant. The error on the test set is $0.09$ which is equal to the error on the training set. Stronger regularization makes the error on the training set larger as the representations collapse. The latter causes the correlation between $\DIone'$ and $e'_3$ to be much lower for ERM+HSIC (\cref{fig:corr_e3d1_HSIC_CMNIST}).


Another interesting phenomenon was observed while training ERM+HSIC with a lower $\beta=0.1$. Initially, the model does not fit the training domains (large $\EGzero$). Over time, $\EGzero$ is gradually transformed into $\EGthree$, $\EGtwo$ and $\EGone$ (\cref{fig:CMNIST_HSIC_beta2}). The overall error always stays above 0.7 during the training, but the composition of the error changes significantly.

On \textbf{Camelyon17}, training domain distinguishability and training-validation domain distinguishability are virtually the same for all algorithms (\cref{fig:Camelyon_DANN,fig:Camelyon_CORAL_HSIC_DRO_SD}). Nevertheless, validation set inseparability is always non-zero, which is the largest contributor to the error on unseen domains for well performing models. The magnitude of $\EGone$ varies a lot during the training and across random seeds (\cref{sec:app-more-results}). Additionally, there is no correlation between $\DIone'$ and $\EGthree'$ (\cref{fig:corr_e3d1_DeepCORAL_Camelyon,fig:corr_e3d1_GroupDRO_Camelyon}). This gives a little hope that focusing on domain invariance will help succeed on this dataset.
\begin{table}[t]
\centering
\begin{tabular}{@{}llcc@{}}
\toprule
Algorithm               & Initialization & Valid. & Test  \\ \midrule
ERM (\cref{fig:Camelyon17_ERM_random_val})                     & Random         & 0.164      & 0.287 \\
ERM                     & Frozen BYOL    & 0.138      & 0.077 \\
ERM (wd=0, \cref{fig:Camelyon17_ERM_frozenBYOL_val})              & Frozen BYOL    & 0.097      & 0.070 \\
SD ($\beta=5$)          & Frozen BYOL    & 0.111      & 0.070 \\
GroupDRO ($\beta=5$)    & Frozen BYOL    & 0.140      & 0.078 \\
IRM ($\beta=5$)         & Frozen BYOL    & 0.147      & 0.081 \\
$\EGtwo'$ (lower bound, est.) & Frozen BYOL    & 0.123      & 0.062 \\ \bottomrule
\end{tabular}
\caption{Generalization error $\EGthree'$ on validation and test domains of several algorithms trained on Camelyon17. Weight decay coefficient for all algorithms is set to $0.01$ except the third row. The last row shows the generalization error without classifier non-invariance component ($\EGtwo'$). As the representations are fixed, these values remain constant during the training and serve as a lower bound for all models with frozen representations. See \cref{sec:tool} on why the estimated lower bound is poor on the validation set.
}
\label{tab:frozenBYOL}
\end{table}

A more detailed analysis uncovers an interesting pattern. As we have only two unseen domains in Camelyon17, we look at both of them during our analysis (violating ``do not look at the test set'' rule). Many methods we tried, including the baseline, increase the invariance after the first few epochs of training. The large variance of $\EGone$ is observed only \textit{after} $\DIone$ gets small (\cref{fig:Camelyon17_ERM_random_val,fig:Camelyon17_ERM_random_test}). Instead, when the domains are well distinguishable, $\EGone$ is relatively stable (although the absolute values are different on the two domains) and is accompanied by $\EGtwo$ and $\EGthree$. In fact, on the test domain, the largest contributor to the generalization error over the first epochs is $\EGthree$. This implies that if we do not push the representations to be domain-invariant, there is a hope to get better generalization on unseen domains by focusing on the classifier. Unfortunately, even the basic ERM, without additional loss terms, converges towards training-domain invariance and harms test set separability.

To verify that focusing on the classifier can be beneficial, we take a ResNet-50 pretrained on ImageNet using BYOL algorithm~\cite{BYOL} and fix the representations. At this point, the representations of different domains are well distinguishable (similar to the first epochs of a regular training), $\EGone'$ 
for the validation and test domains are $0.123$ and $0.062$ respectively. 
This implies that the test domain is more similar to the training domains than the validation domain in BYOL's space. As we freeze the representations, these numbers serve as lower bounds for the generalization error, and the training can affect only $\EGzero$ and $\EGthree$. Experiments show that the ERM baseline with no weight decay (\cref{fig:Camelyon17_ERM_frozenBYOL_val,fig:Camelyon17_ERM_frozenBYOL_test}), as well as SD algorithm with a large $\beta$, can indeed decrease $\EGthree$ and approach the lower bound (\cref{tab:frozenBYOL}). GroupDRO and IRM fail to decrease $\EGthree$, while DANN and ERM+HSIC are not applicable to the setup with fixed representations. This result indicates that avoiding even a little collapse of representations, which in these cases coincides with increased domain invariance, opens new opportunities to decrease the generalization error by targeting just $\EGthree$. The causal link between poor representations of unseen domains and domain invariance is not clear.

\textbf{Model selection.}
To verify whether model selection using validation domains is suitable, we compute the correlation between accuracy of the models on validation and test domains at each epoch for a given dataset/algorithm pair. On Colored MNIST we get $0.98$, $0.8$, $0.73$ and $0.31$ correlations for ERM+HSIC, DeepCORAL, SD and DANN, respectively. On Camelyon17 basically no correlation is observed, which is similar to the findings in \cite{accuracyontheline}. 
Corresponding plots are shown in \cref{sec:appcorrelation}.

\textbf{Model complexity.}
To test the dependence of generalization failures on model complexity, we followed \cite{deep-double-descent} to train several ResNet 18K models ($k = 2,4,8,16,32,64$) with ERM+HSIC objective on Colored MNIST ($\beta=1$). As seen on \cref{fig:CMNIST_HSIC_resnetk}, larger models better fit the training set, but get worse training-test domain invariance ($\DIone'$). 




%% file: sections/conclusion.tex
\section{Conclusion}
In this paper we have developed tools to examine failure modes of domain generalization algorithms. We have identified four components of a generalization error on unseen domains and three components of domain distinguishability. We have identified two distinct patterns for failures. In one case, most algorithms achieved domain invariance on training domains but not on unseen domains which resulted in large generalization errors. In the other case, we showed that increased domain invariance coincides with degradation of learned representations on unseen domains. We showed evidence that without achieving domain invariance it is possible to target one of the failure modes and improve generalization even with the basic ERM algorithm. We hope these findings will catalyze future work on the development of algorithms with better domain generalization capabilities.

%% file: sections/appendix.tex
\section{Proofs}\label{sec:proofs}
This appendix includes the proofs of Propositions \ref{prop:one} and \ref{prop:two}.

\begin{proof}[Proof of Proposition \ref{prop:one}]
Let us fix a label value $y \in \mathcal{Y}$ and two distinct two distinct domains $p_1(x,y)$ and $p_2(x,y)$ from the union of test and training domains. Let $X_1 \sim p_1(x,y), X_2 \sim p_2(x,y), Z_1 = h(X_1),$ and $Z_2 = h(X_2)$. Let $w$ be the parameters of the common classifier for all training and test domains (i.e., the classifier that makes $\EGtwo' = 0$).
Then for all $z$,
\begin{align}
    p_{1}(z | y) &= \frac{p_1(z) p_1(y | z)}{\int p_1(z') p_1(y | z')dz'}\\
    &= \frac{p_2(z) p_1(y | z)}{\int p_2(z') p_1(y | z')dz'} && \text{\hspace{-2.5cm}(by domain inv. of $p(z)$)}\\
    &= \frac{p_2(z) \delta_y\left(\text{argmax}_k f_w(z)_k\right)}{\int p_2(z') \delta_y\left(\text{argmax}_k f_w(z')_k\right)dz'}\\
    &= \frac{p_2(z) p_2(y | z)}{\int p_2(z') p_2(y | z')dz'} = p_2(z|y).
\end{align}
Above $\delta_y(y')$ is the Krokecker's delta function and the last two transitions use the fact the classifier with parameters $w$ correctly classifies representations of both domains.
\end{proof}

\begin{proof}[Proof of Proposition \ref{prop:two}]
Let $X_1 \sim p^1_1(x,y), Z_1 = h(X_1)$ be a random sample from the first training domain and $X_i \sim p^3_i(x,y), Z_i = h(X_i)$ be a random sample from the $i$-th test domain.
Let $w$ be the parameters of the classifier found by the training algorithm that classifies correctly representations of all training domains.
If we fix a label $y\in\mathcal{Y}$, then
\begin{align*}
\mathbb{P}(\text{argmax}_k f_w(Z_i)_k = y \mid Y_i = y) =\\
\quad=\mathbb{P}(\text{argmax}_k f_w(Z_1)_k = y \mid Y_1 = y),
\numberthis\label{eq:prop2}
\end{align*}
as $p^3_i(z | y) = p^1_1(z | y)$ by the assumptions.
As $\EGzero'=0$ implies that the classifier $f_w$ perfectly classifies representations from the first training domain, the right-hand-side of (\ref{eq:prop2}) will be equal to 1.
Therefore, the $f_w$ will correctly classify also the representations of test domain $p^3_i(x,y)$. As $i$ was arbitrary, this concludes the proof.
\end{proof}

\section{Decompositions}\label{sec:decomposition-statements}
This appendix describes additional properties of the decompositions presented in the main text (see \cref{sec:decompositions}).

Let $\mathscr{D} = \{p_i(x,y)\}_{i=1}^{n_1+n_3}$ be a family of $n_1 + n_3$ domains.
Let $S$ be a subset of $\{1,2,\ldots,n_1+n_3\}$ of size $n_1$, chosen uniformly at random, and let $\bar{S} \triangleq (\{1,2,\ldots,n_1+n_3\}\setminus S)$ be the complement of $S$.
This subset $S$ defines a partition of $\mathscr{D}$ into training and test domains.
Let $w(S)$ and $\theta(S)$ denote the parameters of the classification head and the feature extractor after training on domains specified by $S$.
Let $(X_1,Y_1),\ldots,(X_{n_1+n_3}, Y_{n_1+n_3})$ be random variables drawn from distributions $p_1(x,y),\ldots,p_{n_1+n_3}(x,y)$, respectively.
To simplify the derivations below, we define
\begin{equation}
    F(A; w, \theta) = \frac{1}{|A|}\sum_{i \in A} \E_{X_i, Y_i}\sbr{ \ell\rbr{f_w(h_\theta(X_i)), Y_i}}, \label{eq:F}
\end{equation}
for any subset $A$ of $\{1,2,\ldots,n_1+n_3\}$.
To avoid unnecessary technical complications in statements and proofs below, we assume that the feature extractor parameters $\theta$, label classifier parameters $w$, and domain classifier parameters $\theta$ belong to compact domains.
This will allow us to replace infimum operators in the definitions of generalization and invariance metrics by minimum operators.

With these conventions, the generalization metrics can be written the following way
\begin{align}
\EGzero'(S) &= F(S; w(S), \theta(S)),\\
\EGone'(S) &= \min_{w'\in \mathcal{W}} F(\bar{S}; w', \theta(S))\\
\EGtwo'(S) &= F(\bar{S}; \tilde{w}(S),\theta(S)), \text{ where } \tilde{w}(S) \in\arg\min_{w' \in \mathcal{W}} F(S\cup\bar{S}; w',\theta(S)),\label{eq:app-e2}\\
\EGthree'(S) &= F(\bar{S}; w(S),\theta(S)).
\end{align}

Next, we prove several statements to establish the relationship between the metrics.
\begin{proposition}
Assume that $n_1=n_3$ and $(w(S),\theta(S)) \in \arg\min_{w',\theta'}F(w',S,\theta')$. Then
$\E_S\sbr{e'_0(S)} \le \E_S\sbr{e'_1(S)}$.
\label{prop:e0e1}
\end{proposition}
\begin{proof}
\begin{align*}
\E_S\sbr{e'_0(S)} &= \E_S F(S; w(S),\theta(S))\\
&\le \E_{S} \min_{w' \in \mathcal{W}} F(S;w',\theta(\bar{S}))&&\text{(as $(w(S), \theta(S))$ is a global minimum of $F(S; w', \theta')$)}\\
&= \E_{\bar{S}}\min_{w'\in\mathcal{W}} F(\bar{S}; w', \theta(S))&&\text{(replacing $S$ by $\bar{S}$, as $n_1 = n_3 \Rightarrow S \stackrel{d}{=} \bar{S}$)}\\
&=\E_S\sbr{e'_1(S)}.&&\text{(by definition)}
\end{align*}
\end{proof}

Note that the assumptions of the proposition are critical. If $n_3 < n_1$, then the classification on the test domains might work well simply because there are fewer domains in the test set. In an extreme case, when $n_3=1$, the representations of that single domain might be easily separable, while finding a universal good classifier for all training domains can be hard. As the training is performed on the training domains, it is reasonable to assume the learned parameters $(w(S),\theta(S))$ are optimal with respect to the training domains. If this assumption is violated, one cannot exclude that the model can end up with a better classifier for the test domains ($\bar{S}$) while being trained on the training domains ($S$). Even with these assumptions, the partition $\mathscr{D}$ of into training and test domains can be ``adversarial'' in a sense that the test domains are much easier, which will cause smaller $\EGone$. However, such scenario cannot happen for every partition of $\mathscr{D}$. Hence, the Proposition \ref{prop:e0e1} proves the desired relation between $e_0(S)$ and $e_1(S)$ only in expectation over $S$.

The next two inequalities can be proved with weaker assumptions.

\begin{proposition} Let the generalization metrics $e'_1(S)$, $e'_2(S)$ and $e'_3(S)$ be defined as above. Then, (i) $e'_1(S) \le e'_{2}(S)$.
Furthermore, if $w(S) \in \arg\min_{w'} F(S;w',\theta(S))$ then (ii) $e'_2(S) \le e'_{3}(S)$.
\end{proposition}
\begin{proof}
\textbf{(i)} We have that
\begin{align}
\EGone'(S) &= \min_{w'\in \mathcal{W}} F(\bar{S}; w', \theta(S)) \le F(\bar{S}; \tilde{w}(S),\theta(S)) = \EGtwo'(S).
\end{align}

\textbf{(ii)} Let $\alpha \triangleq n_1 / (n_1 + n_3)$. As $\tilde{w}(S)$ is a global minimum of $F(S \cup \bar{S}; w', \theta(S))$, we have that
\begin{equation}
    F(S \cup \bar{S}; \tilde{w}(S), \theta(S)) \le F(S \cup \bar{S}; w(S), \theta(S)),
\end{equation}
which is equivalent to
\begin{equation}
\alpha F(S; \tilde{w}(S), \theta(S)) + (1-\alpha) F(\bar{S}; \tilde{w}(S),\theta(S)) \le \alpha F(S;w(S),\theta(S)) + (1-\alpha)F(\bar{S}; w(S),\theta(S)).
\end{equation}
This simplifies to
\begin{equation}
F(\bar{S}; \tilde{w}(S),\theta(S)) \le F(\bar{S}; w(S),\theta(S)) + \frac{\alpha}{1-\alpha}\rbr{F(S;w(S),\theta(S)) - F(S; \tilde{w}(S), \theta(S))}.
\label{eq:27}
\end{equation}
The additional assumption that $w(S) \in \arg\min_{w'} F(S;w',\theta(S))$ implies that  $F(S;w(S),\theta(S)) \le F(S; \tilde{w}(S), \theta(S))$.
Connecting this with (\ref{eq:27}) we get
\begin{equation}
F(\bar{S}; \tilde{w}(S),\theta(S)) \le F(\bar{S}; w(S),\theta(S)),
\end{equation}
which is the same as $e'_2(S) \le e'_3(S)$.
\end{proof}

Next we investigate the relationship between invariance metrics. Similar to the function $F(A; w, \theta)$ defined above, we define
\begin{equation}
G(A; \omega, \theta) = \frac{1}{|A|}\sum_{i \in A} \E_{X_i}\sbr{ \ell\rbr{g_\omega(h_\theta(X_i)), i}},
\end{equation}
for any subset $A$ of $\{1,2,\ldots,n_1+n_3\}$.
With these conventions, the invariance metrics can be written as follows:
\begin{align}
\DIzero'(S) &=
1 - \min_{\omega \in \Omega} G(S;\omega,\theta(S)) - \frac{1}{n_1},\\
\DIone'(S) &= 1 - \min_{\omega \in \Omega}G(S\cup\bar{S};\omega,\theta(S)) - \frac{1}{n_1+n_3},\\
\DItwo'(S) &= 1- \frac{1}{C}\sum_{y=1}^C \min_{\omega \in \Omega}\E\sbr{\frac{1}{n_1 + n_3}\sum_{i \in S \cup \bar{S}} \ell(g_{\omega}(h_\theta(X_i)), i)\ \bigg\lvert\ E_y} -  \frac{1}{n_1 + n_3},
\end{align}
where $E_y$ denotes the event $(Y_1=y \wedge \cdots \wedge Y_{n_1+n_3}=y)$.

The relationship between $\DIzero'(S)$ and $\DIone'(S)$ is significantly more complicated as the sets of domains are different, and the accuracy scores are not directly comparable. 
Furthermore, the training algorithm can be ``adversarial'' in the sense that its produced representations of training domains are more distinguishable compared to that of testing domains.
The relationship between $\DIone'(S)$ and $\DItwo'(S)$ is simpler, and the desired inequality can be proved with mild assumptions.
\begin{proposition}
Let $d_1(S)$ and $d_2(S)$ be defined as above. Assuming that $P(Y_i=y)=1/C$ for all $i\in\{1,2,\ldots,n_1+n_3\}$ and $y \in \{1,2,\ldots,C\}$, we have that $d'_1(S) \le d'_2(S)$.
\end{proposition}
\begin{proof}
\begin{align*}
\DItwo'(S) &= 1- \frac{1}{C}\sum_{y=1}^C \min_{\omega \in \Omega}\E\sbr{\frac{1}{n_1 + n_3}\sum_{i \in S \cup \bar{S}} \ell(g_{\omega}(h_\theta(X_i)), i)\ \bigg\lvert\ E_y} -  \frac{1}{n_1 + n_3}\\
&= 1- \frac{1}{C}\sum_{y=1}^C \min_{\omega \in \Omega}\rbr{\frac{1}{n_1 + n_3}\sum_{i \in S \cup \bar{S}}\E_{X_i}\sbr{\ell(g_{\omega}(h_\theta(X_i)), i)\ \bigg\lvert\ Y_i = y}} -  \frac{1}{n_1 + n_3}\\
&\ge 1- \min_{\omega \in \Omega}\rbr{\frac{1}{C}\sum_{y=1}^C\frac{1}{n_1 + n_3}\sum_{i \in S \cup \bar{S}}\E_{X_i}\sbr{\ell(g_{\omega}(h_\theta(X_i)), i)\ \bigg\lvert\ Y_i = y}} -  \frac{1}{n_1 + n_3}\\
&= 1- \min_{\omega \in \Omega}\rbr{\frac{1}{n_1 + n_3}\sum_{i \in S \cup \bar{S}}\E_{X_i}\sbr{\ell(g_{\omega}(h_\theta(X_i)), i)}} -  \frac{1}{n_1 + n_3}\\
&= 1 - \min_{\omega \in \Omega}G(S\cup\bar{S};\omega,\theta(S))-\frac{1}{n_1+n_3}\\
&=d'_1(S).
\end{align*}
\end{proof}

\section{Dataset Samples}
\label{sec:app-dataset-samples}
\cref{fig:dataset-samples} shows one sample per class per domain for Colored MNIST and Camelyon17 datasets used in our experiments. Note that Camelyon17 images are originally four-channel RGBA images, while Colored MNIST images have 50 ``channels'', the first three of which are interpreted as RGB in the figure.

\section{Hyperparameters of the algorithms}
\label{sec:hyperparams}
For all algorithms we fixed the regularization strength hyperparameter space to $\{0.1, 0.5, 1, 2, 5, 10, 15\}$. We used SGD optimizer in all cases. For Colored MNIST dataset we used learning rate of 0.01 and for Camelyon17 we used 0.001. All methods used 0.01 weight decay unless explicitly noted (\eg in \cref{tab:frozenBYOL}). We trained all our algorithms for fixed 10 epochs. One epoch of training on Camelyon17 took about 30 minutes (some algorithms, \eg DANN, IRM, take a bit longer because of more complicated computations) on our machine with two NVIDIA Titan V GPUs,  while the algorithms on Colored MNIST took only several seconds.

To produce the plots used in this paper we extracted learned representations by forwarding the data through the learned network and trained multiple logistic regression functions. The duration of this process strongly depends on the size of the representations and the number of samples in the dataset. Processing the entire dataset for one model on a single Titan V GPU takes 20 minutes for Camelyon17 and less than a minute for Colored MNIST. To produce a single plot we process 10 models (after each epoch) or 8 models (for each value of the hyperparameter $\beta$). For some plots, we processed 10 models per epoch epochs to see the behavior in more details in the first three epochs (\eg \cref{fig:Camelyon17_ERM_random_test}). 
    
\section{More Results}
\label{sec:app-more-results}

In this section we provide more details about the trained models. We had trouble training IRM on both datasets. For many combinations of hyperparameters the loss was getting NaN at early stages of the training. One of the successful attempts was on Colored MNIST with $\beta=10$. The biggest contributor to the error was changing during the training, as shown on \cref{fig:CMNIST_IRM}. 

\cref{fig:CMNIST_additional} shows generalization errors and domain-distinguishability on Colored MNIST for SD and GroupDRO algorithms. Still, $\EGtwo$ and $\EGthree$ are the main contributors to the generalization error. \cref{fig:CMNIST_HSIC_resnetk} shows that more complex networks (i.e. ResNets with more parameters) have smaller errors on the training set and less invariance across training and test domains.

\cref{fig:Camelyon_CORAL_HSIC_DRO_SD} shows the performance of three more algorithms on Camelyon17. With strong enough regularization, the models underfit the training set. The biggest contributor to the generalization error for other models is the test set inseparability. \cref{fig:Camelyon_DeepCORAL_seed} shows the impact of the random seed when DeepCORAL is trained on Camelyon17. Some variance is visible in both generalization error and domain distringuishability. In all three cases the algorithm increases domain invariance by the end of the training, but it does not translate into better generalization. 

\section{Notes on Label Shift}
\label{sec:app-label-shift}
Although label shift is common in real-world applications, it makes error analysis extremely complicated.
First, when the number of domains is large, characterizing types of label shift is a challenge on its own. 
Label shift can appear between training domains, between individual training and test domains, and between the union of training domains and the union of test domains. iWildCam dataset from WILDS benchmark has all of these shifts.
Second, the accuracy of the model can change in both directions on the test domains under label shift. In case when the test set contains more examples from ``easier'' classes, the error on the test set might be low, and it can hide the performance drop due to the domain shift.
We also observe that most of the current domain generalization algorithms simply ignore the label shift issue. In fact, the proof of Proposition \ref{prop:one} 
shows that if $e'_2=0$ and the representations are domain invariant with respect to the union of training and test domains, label distribution $p(y)$ is also invariant. It implies that enforcing invariance of representations $p(z)$ when $p(y)$ is not invariant is not desirable, as 
a common classifier will not exist. 
Finally, to the best of our knowledge, robustness of algorithms with respect to unknown label shifts is not explored even if $p(x|y)$ is constant across domains. Current label shift literature (\eg \cite{label-shift-Lipton}) assumes access to the test domain and discusses adaptation strategies.


\section{Visualization of Representation Spaces}
\label{sec:app-representation-space}
To visualize the learned representations, we forward pass the datasets (all domains) through the networks, take the representations $z=h_{\theta}(x)$, perform a single two-dimensional Principal Component Analysis for each model (on the combination of all domains), and plot the results. Colors encode the labels, while the marker type encodes the domain. \cref{fig:z_CMNIST_ERM} shows a typical failure on Colored MNIST dataset. For example, the representations of digits $1$ are grouped in multiple clusters. Moreover, the representations of the samples from the validation domain $V_1^2$ (indicated by star markers) are quite far from the other clusters. This shows that a linear classifier that successfully works on the training domains might not be able to classify the digits in the validation domain. \cref{fig:z_CMNIST_HSIC1} shows the more successful case obtained with ERM+HSIC algorithm where the representations of digits are grouped regardless of the domains.

\cref{fig:z_Camelyon_ERM} shows the space learned by ERM baseline on Camelyon17 dataset. It demonstrates a case when the dots are clearly separable according to the label (color) for the training domains, but there are a few dots from the validation and test domains ($V_1^2$ and $V_1^3$) that are in the neighborhood of the wrong color. On the 2D space it can be seen that the representations of the samples from validation and test domains are not fully separable.  
\cref{fig:z_Camelyon_HSIC15} demonstrates the case when the regularization is too strong and the representations of most samples have collapsed near $(0,0)$, leading to large $\EGzero$.

\section{Correlation Plots}\label{sec:appcorrelation}

The first row of \cref{fig:corr_val_test} shows correlation plots between generalization error $\EGthree'$ and training-validation domain distinguishability $\DIone'$ for various algorithms. Every dot corresponds to a single model saved at each epoch. We expect to see positive correlation between these metrics. For Colored MNIST, DeepCORAL shows $0.52$ correlation (\cref{fig:corr_e3d1_DeepCORAL_CMNIST}), while ERM+HSIC has $0.05$ correlation (\cref{fig:corr_e3d1_HSIC_CMNIST}). The latter is explained by large $\EGzero$ when regularization is too strong. We do not see any correlation between $\EGthree'$ and $\DIone'$ on Camelyon 17. 
The second row shows the relationship between the errors of many models on validation and test sets. High correlation implies that the accuracy on the validation set can serve as a good metric for model selection. ERM+HSIC, the only algorithm that performed well on Colored MNIST, demonstrates high correlation (\cref{fig:corr_val_test_HSIC_CMNIST}. The correlation for other algorithms is lower. \cref{fig:corr_val_test_CORAL_Camelyon} shows no correlation for DeepCORAL trained on Camelyon17 with various seeds.

\begin{figure}[!ht]
    \centering
    \begin{subfigure}{0.37\linewidth}
    \includegraphics[width=\linewidth]{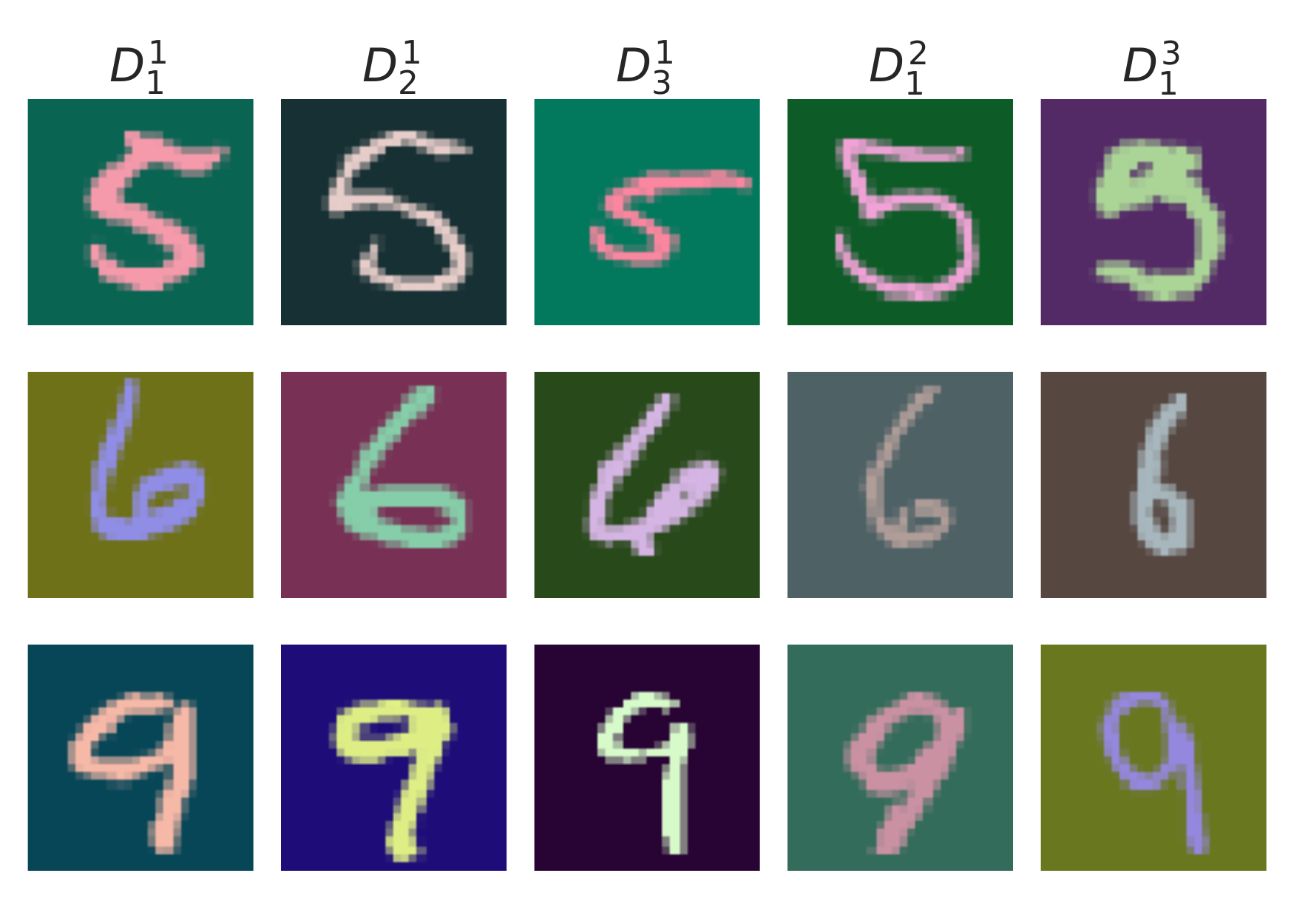}
    \caption{Colored MNIST}
    \end{subfigure}
    \begin{subfigure}{0.58\linewidth}
    \includegraphics[width=\linewidth]{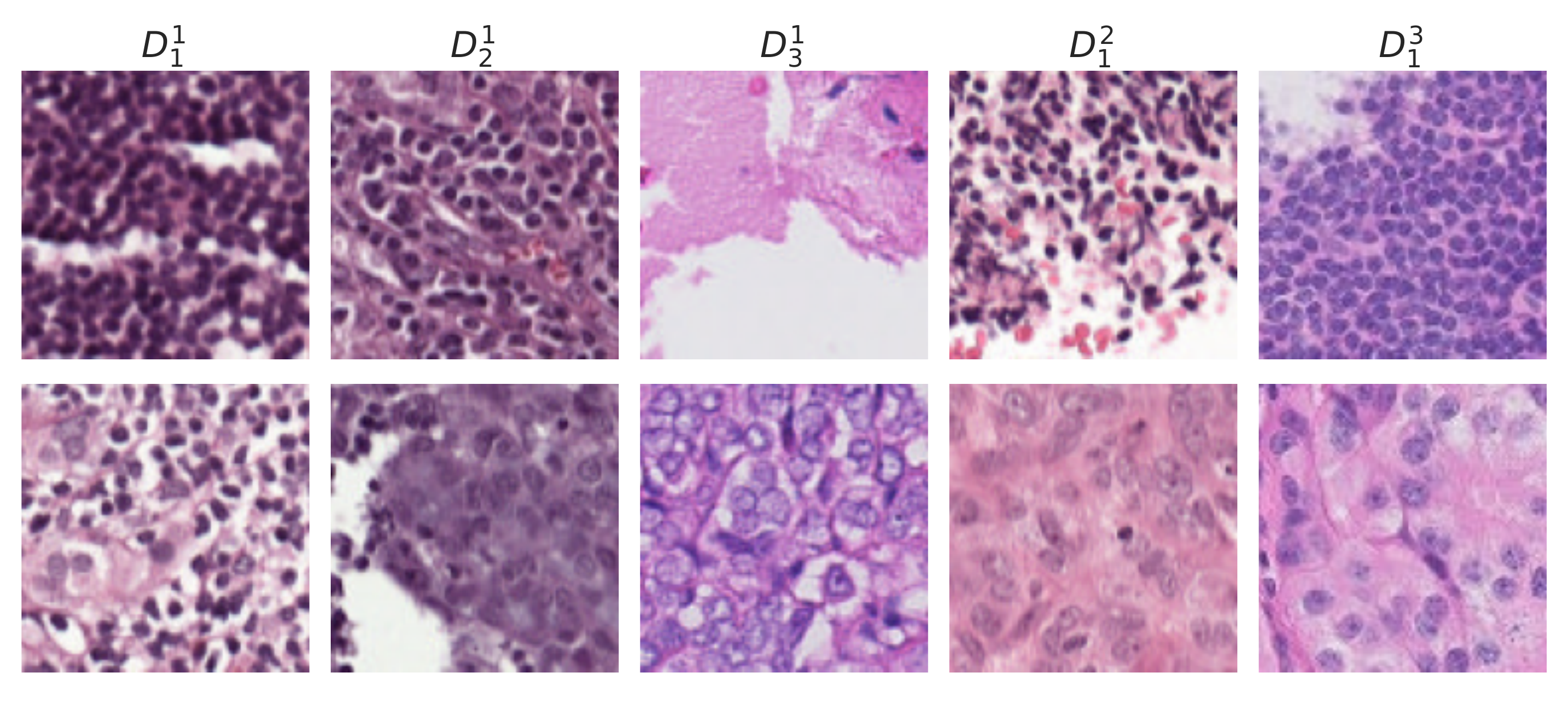}
    \caption{Camelyon17}
    \end{subfigure}
    \caption{Samples of Colored MNIST and Camelyon17 datasets. Both datasets have three training domains $D_1^1$, $D_2^1$, $D_3^1$, one validation domain $D_1^2$ and one test domain $D_1^3$. The first row of (b) shows normal tissue samples, while the samples of the second row contain tumor.}
    \label{fig:dataset-samples}
\end{figure}

\begin{figure}[!ht]
\centering
\begin{subfigure}{0.21\linewidth}
    \includegraphics[width=\linewidth]{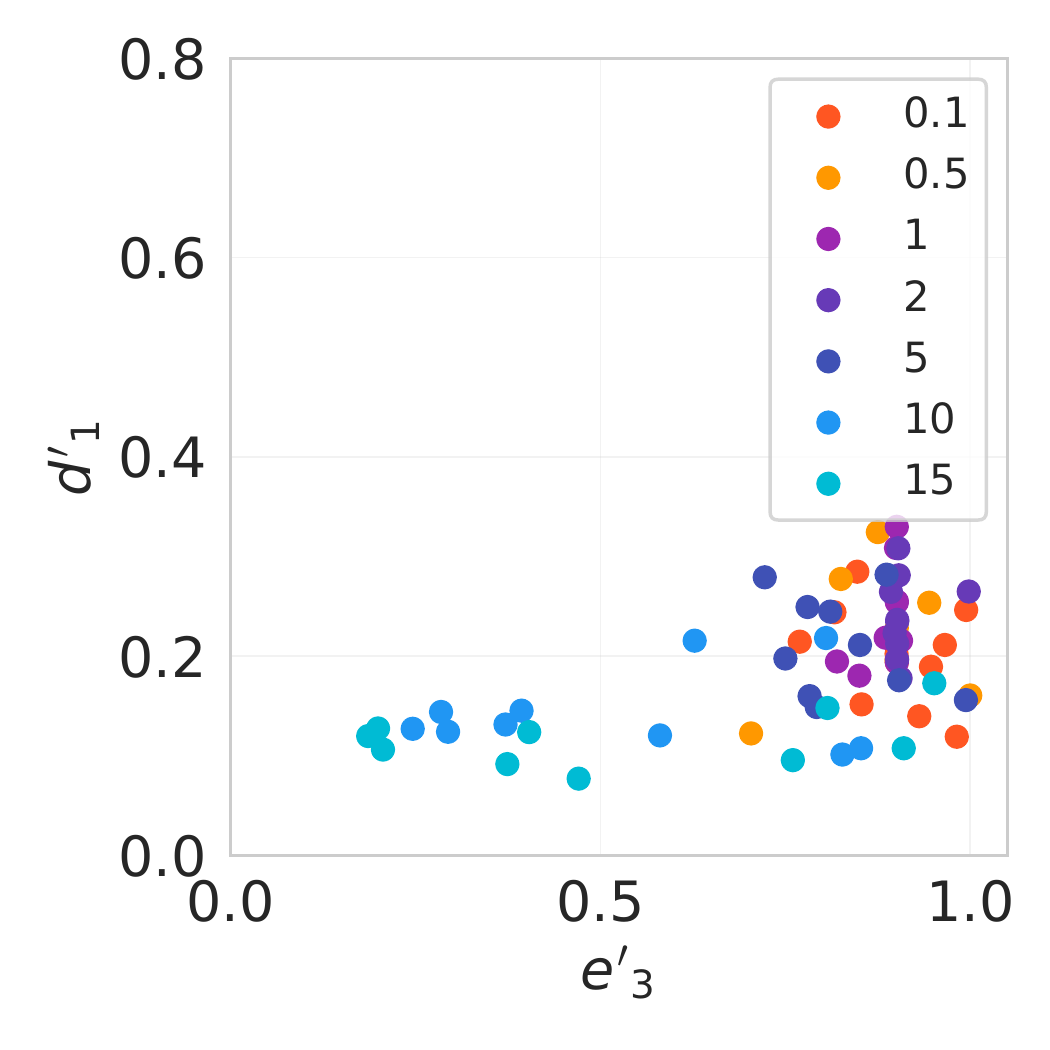}
    \caption{DeepCORAL}
    \label{fig:corr_e3d1_DeepCORAL_CMNIST}
\end{subfigure}
\begin{subfigure}{0.21\linewidth}
    \includegraphics[width=\linewidth]{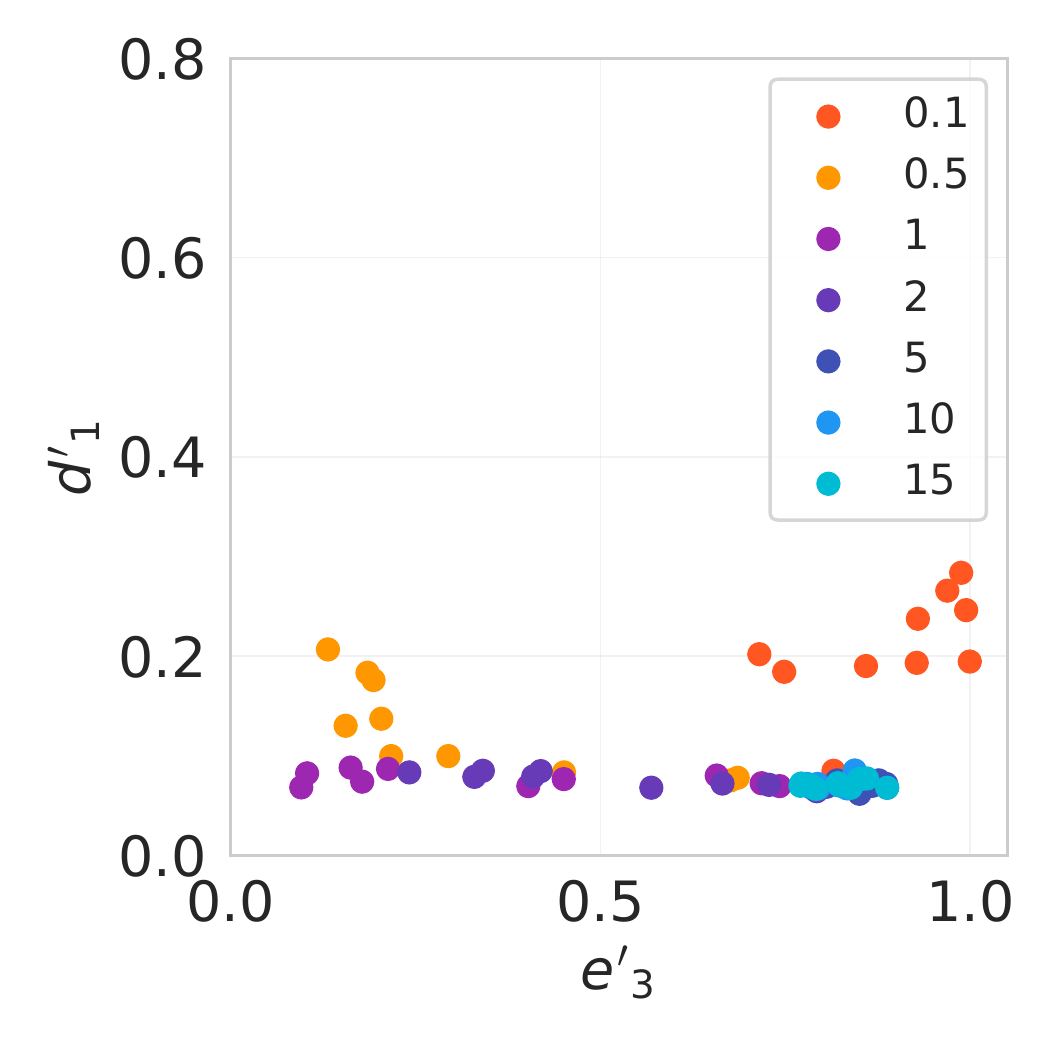}
    \caption{ERM+HSIC}
    \label{fig:corr_e3d1_HSIC_CMNIST}
\end{subfigure}
\begin{subfigure}{0.21\linewidth}
    \includegraphics[width=\linewidth]{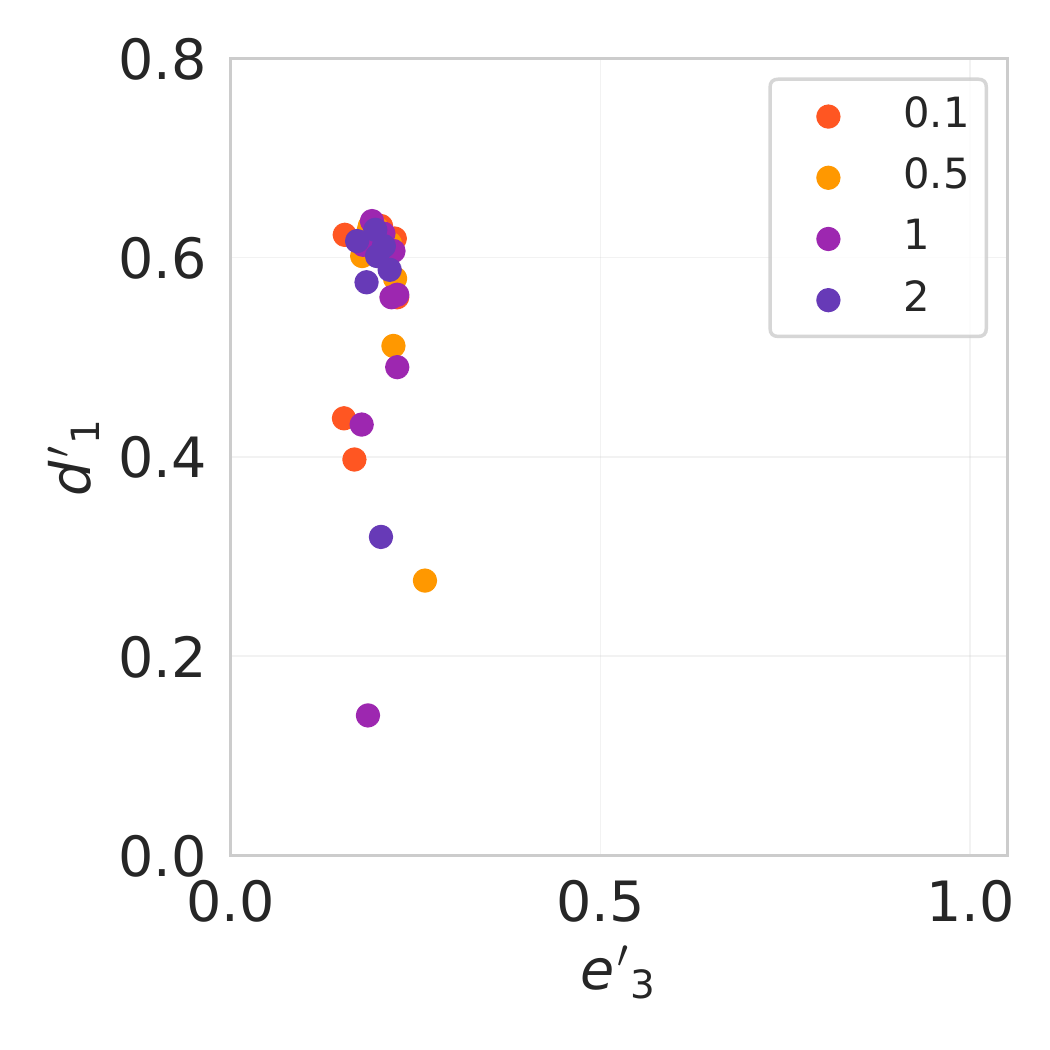}
    \caption{DeepCORAL}
    \label{fig:corr_e3d1_DeepCORAL_Camelyon}
\end{subfigure}
\begin{subfigure}{0.21\linewidth}
    \includegraphics[width=\linewidth]{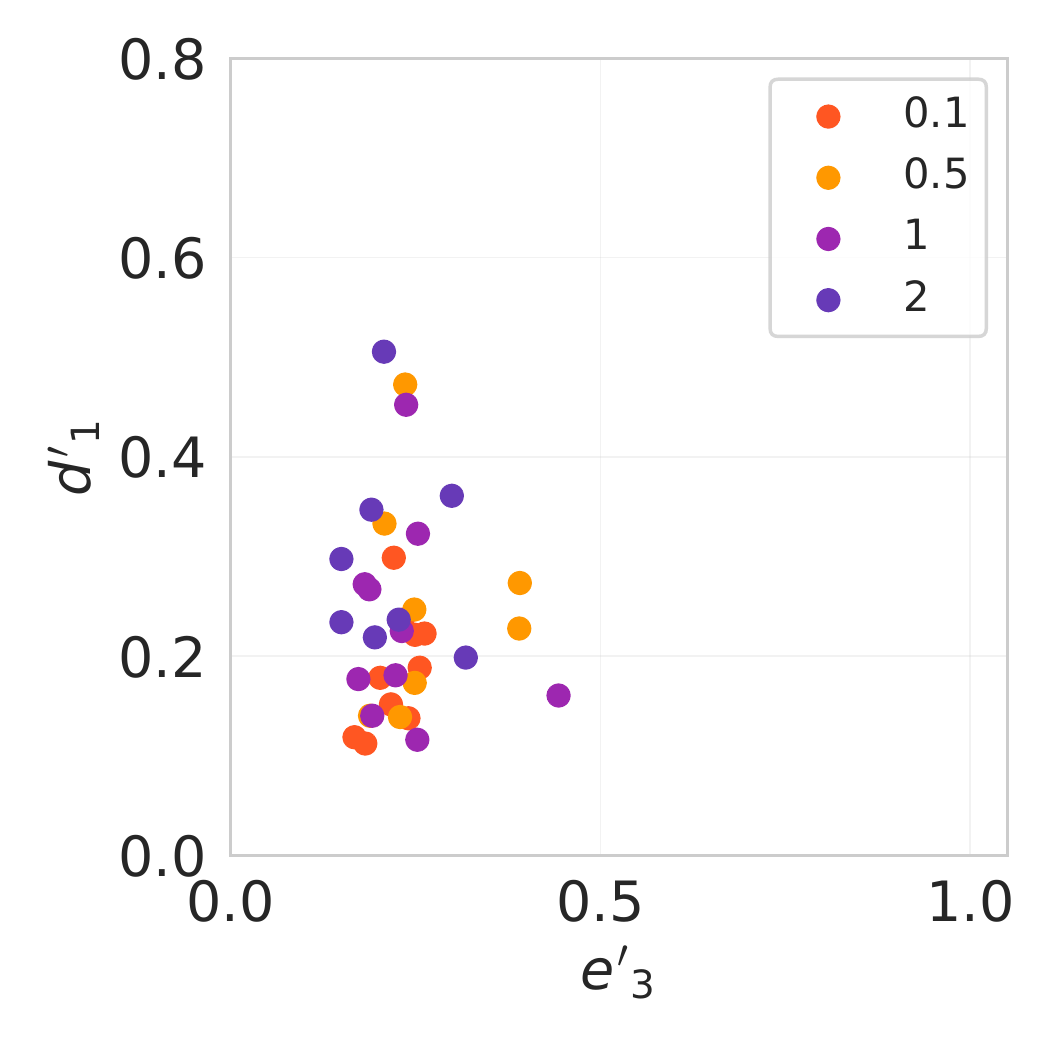}
    \caption{GroupDRO}
    \label{fig:corr_e3d1_GroupDRO_Camelyon}
\end{subfigure}

\begin{subfigure}{0.21\linewidth}
    \includegraphics[width=\linewidth]{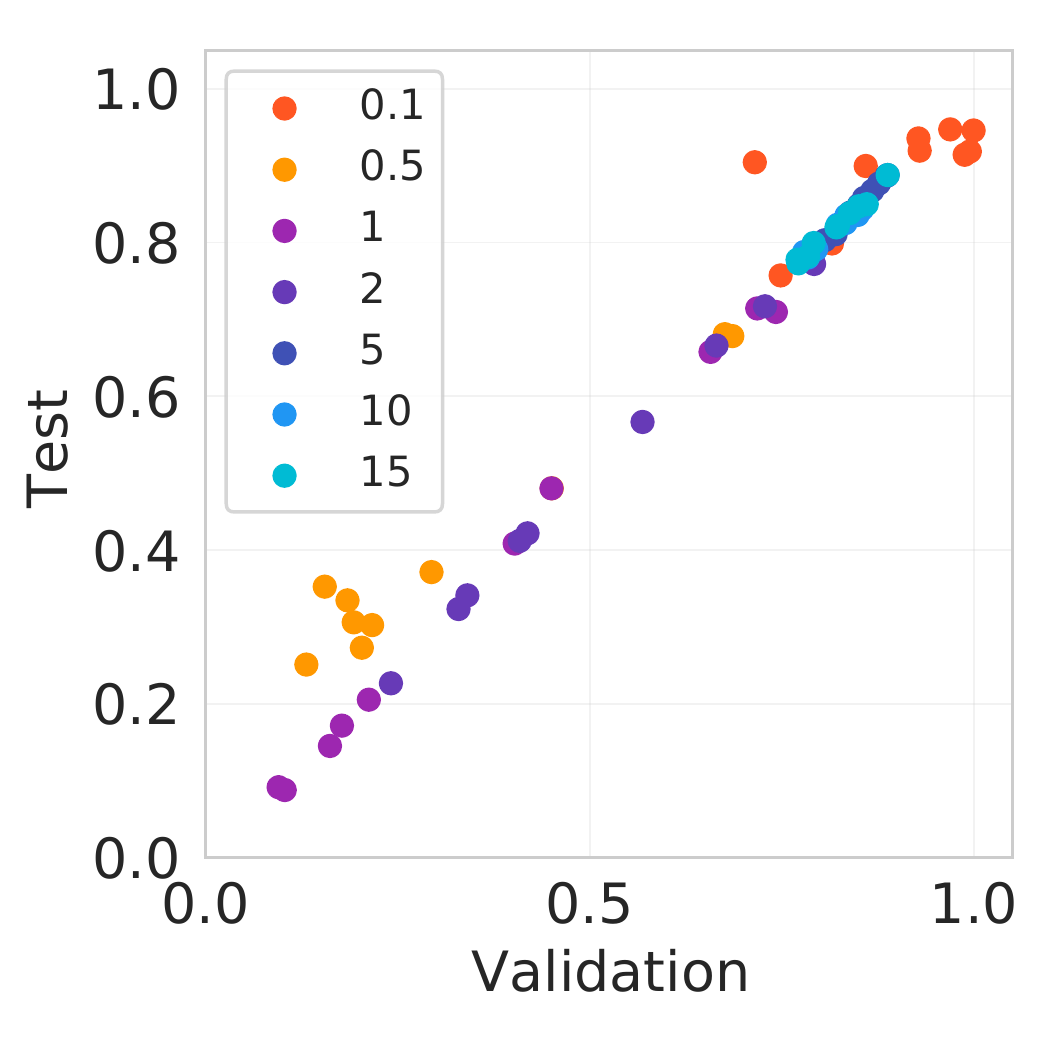}
    \caption{ERM+HSIC}
    \label{fig:corr_val_test_HSIC_CMNIST}
\end{subfigure}
\begin{subfigure}{0.21\linewidth}
    \includegraphics[width=\linewidth]{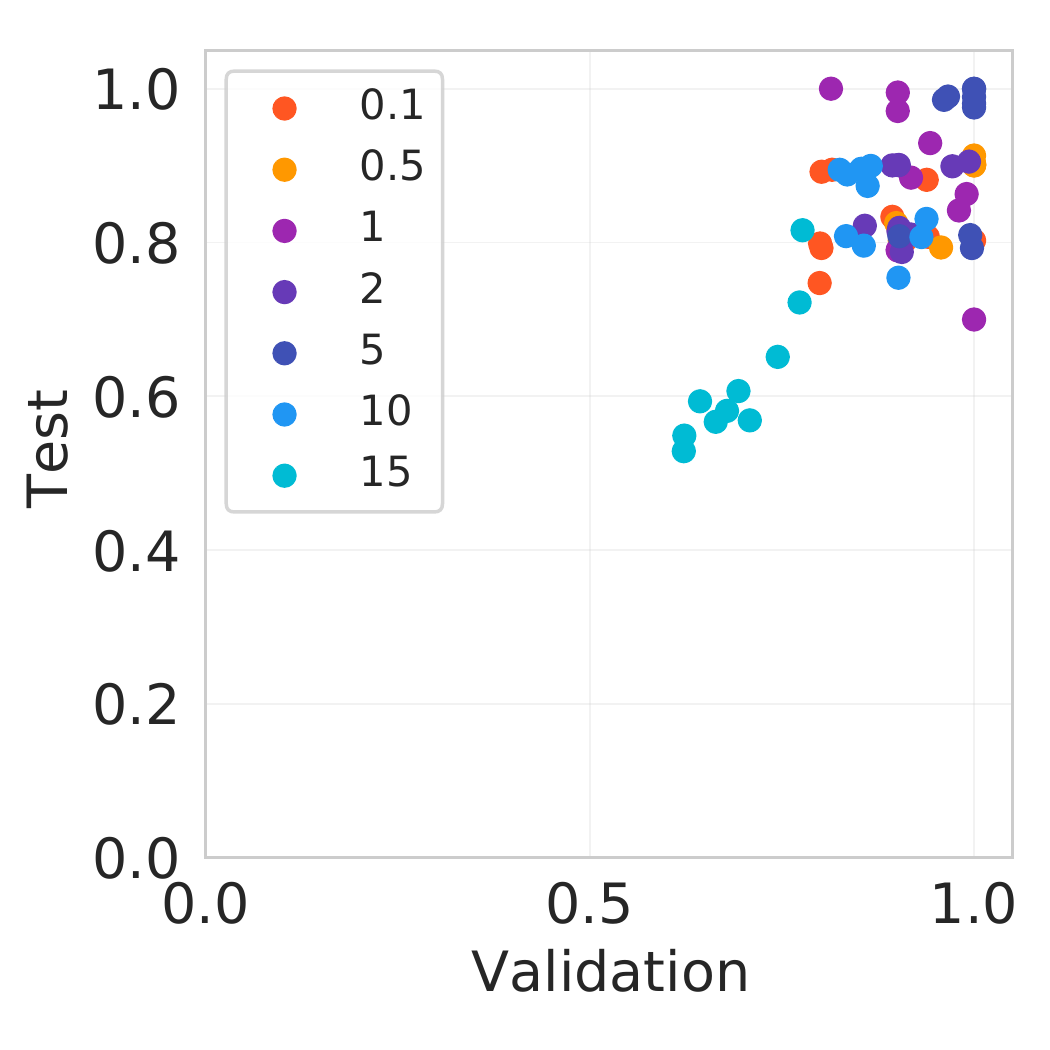}
    \caption{SD}
    \label{fig:corr_val_test_SD_CMNIST}
\end{subfigure}
\begin{subfigure}{0.21\linewidth}
    \includegraphics[width=\linewidth]{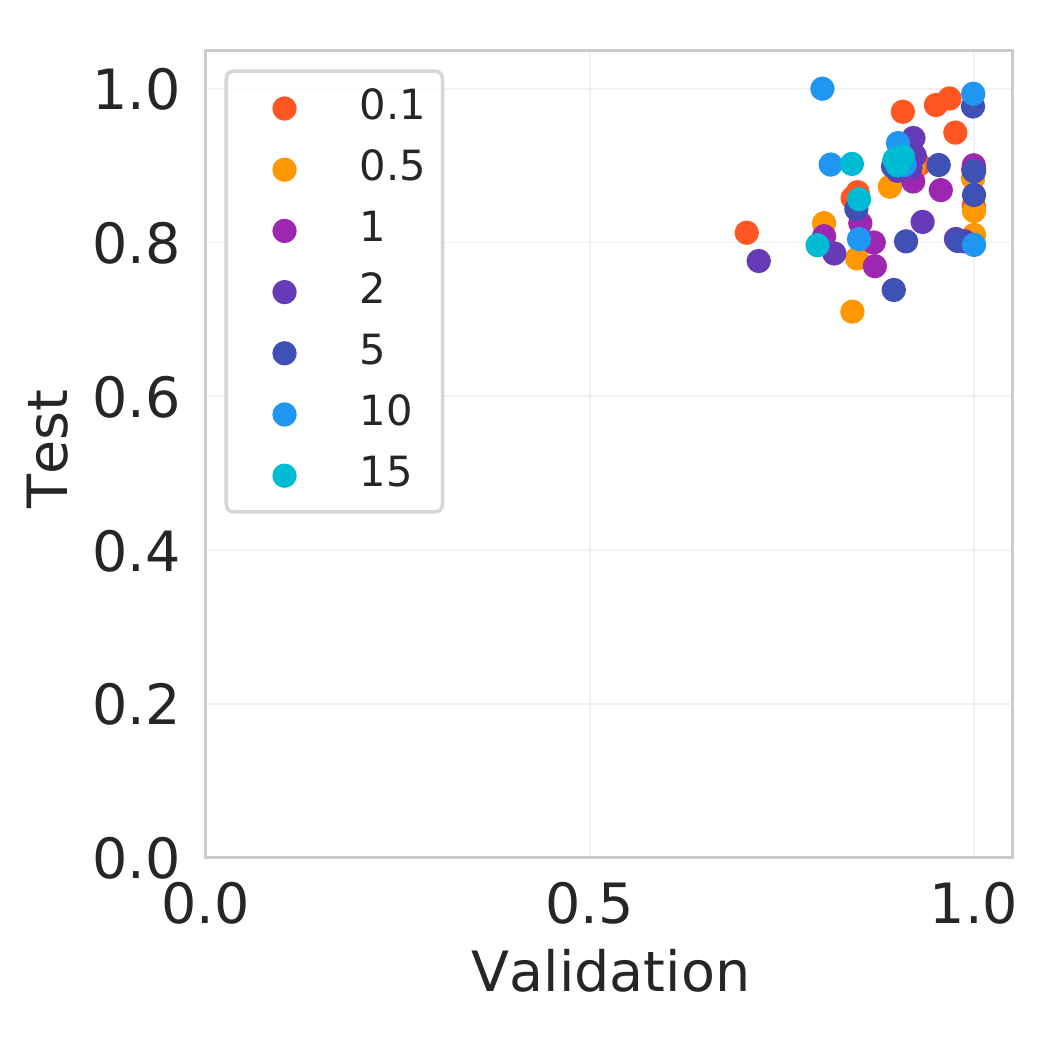}
    \caption{DANN}
    \label{fig:corr_val_test_DANN_CMNIST}
\end{subfigure}
\begin{subfigure}{0.21\linewidth}
    \includegraphics[width=\linewidth]{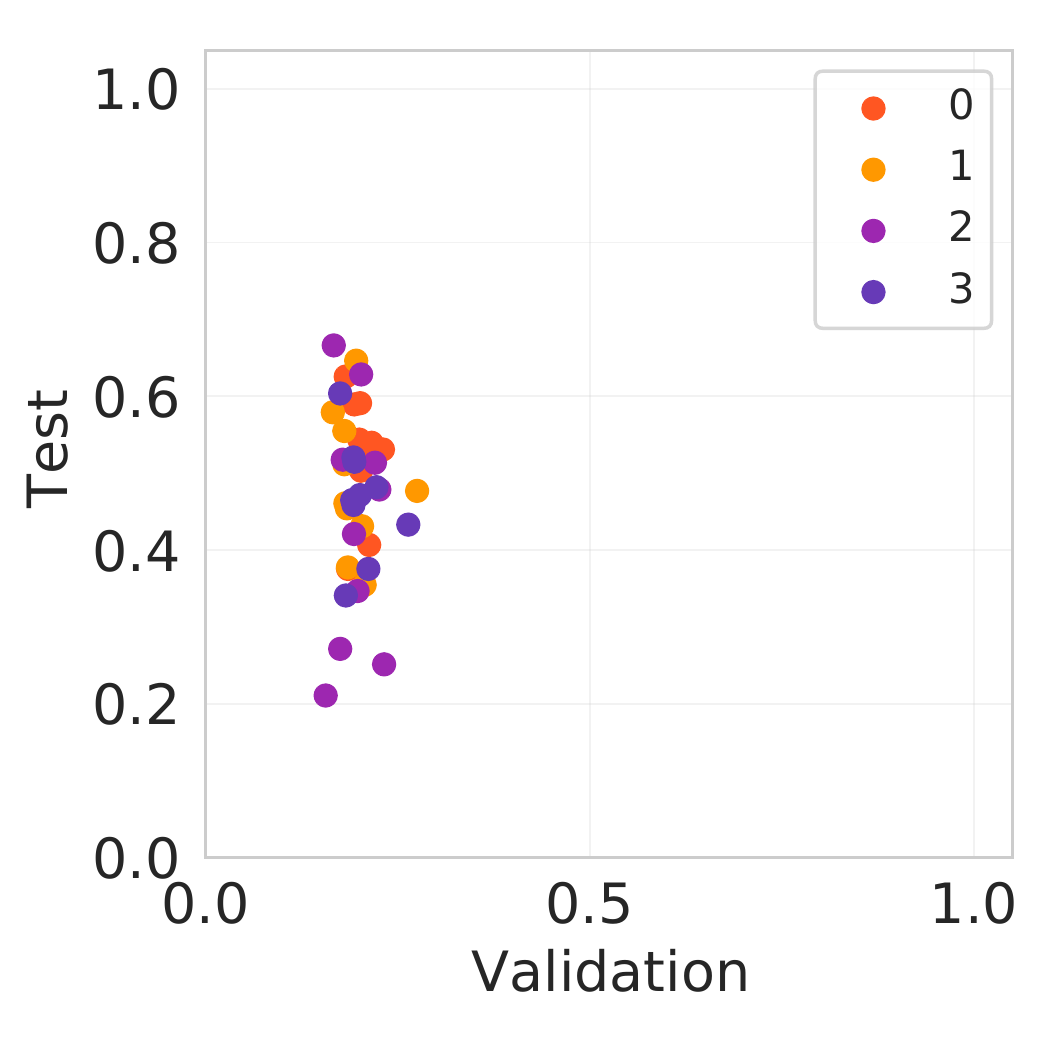}
    \caption{DeepCORAL}
    \label{fig:corr_val_test_CORAL_Camelyon}
\end{subfigure}
\caption{First row: relationship between generalization error $\EGthree'$ and training-validation domain distinguishability $\DIone'$ for various algorithms on (a-b) Colored MNIST and (c-d) Camelyon 17 datasets. Colors of the dots indicate the strength of regularization. Second row: correlation plots between validation and test errors. Colors of the dots indicate the strength of regularization for the models trained on Colored MNIST (e-g) or random seed for the models trained on Camelyon17 (h).}
\label{fig:corr_val_test}
\end{figure}

\begin{figure}
\centering
\begin{subfigure}{0.23\linewidth}
    \includegraphics[width=\linewidth]{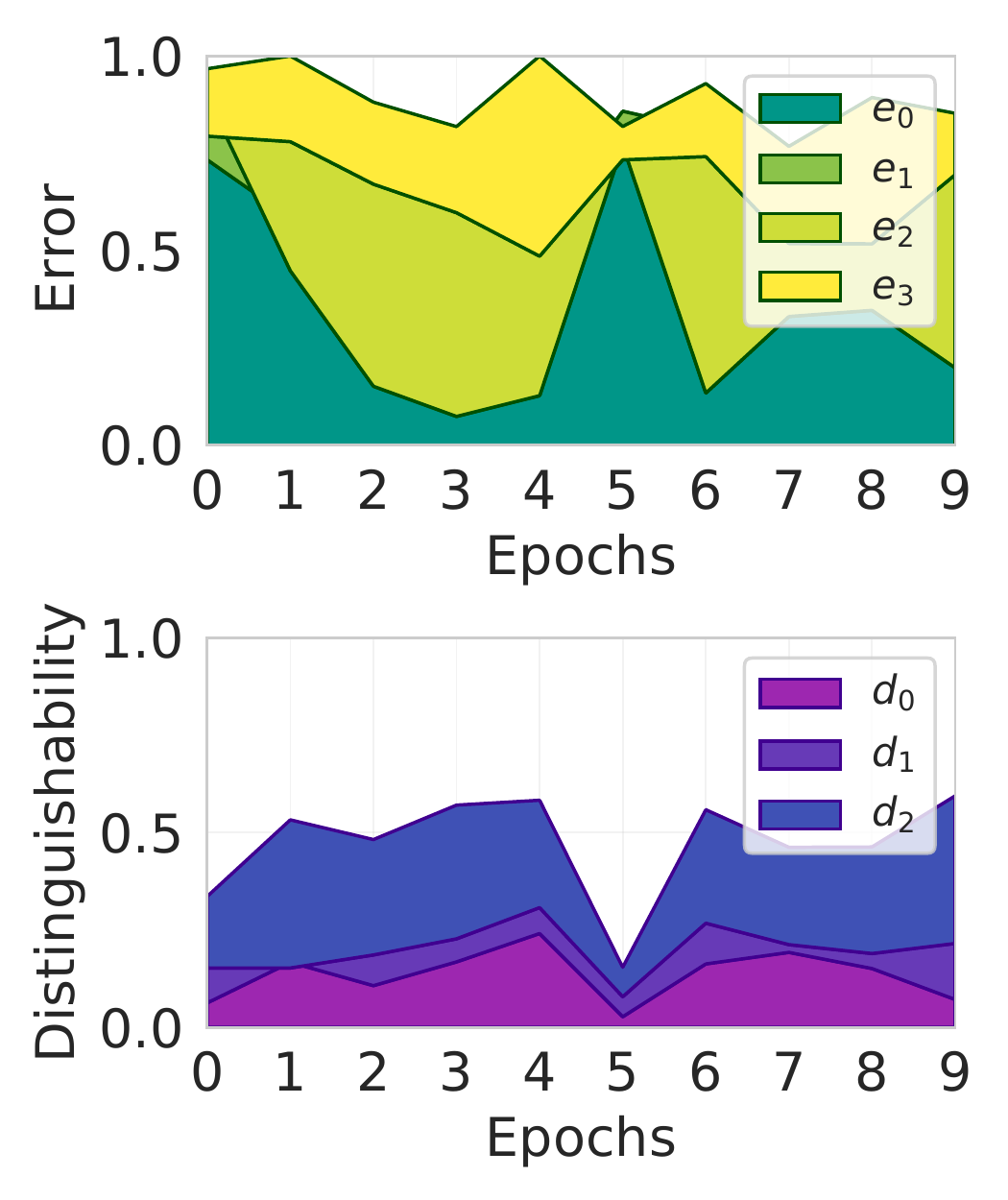}
    \caption{IRM ($\beta=10$)}
    \label{fig:CMNIST_IRM}
\end{subfigure}
\begin{subfigure}{0.23\linewidth}
    \includegraphics[width=\linewidth]{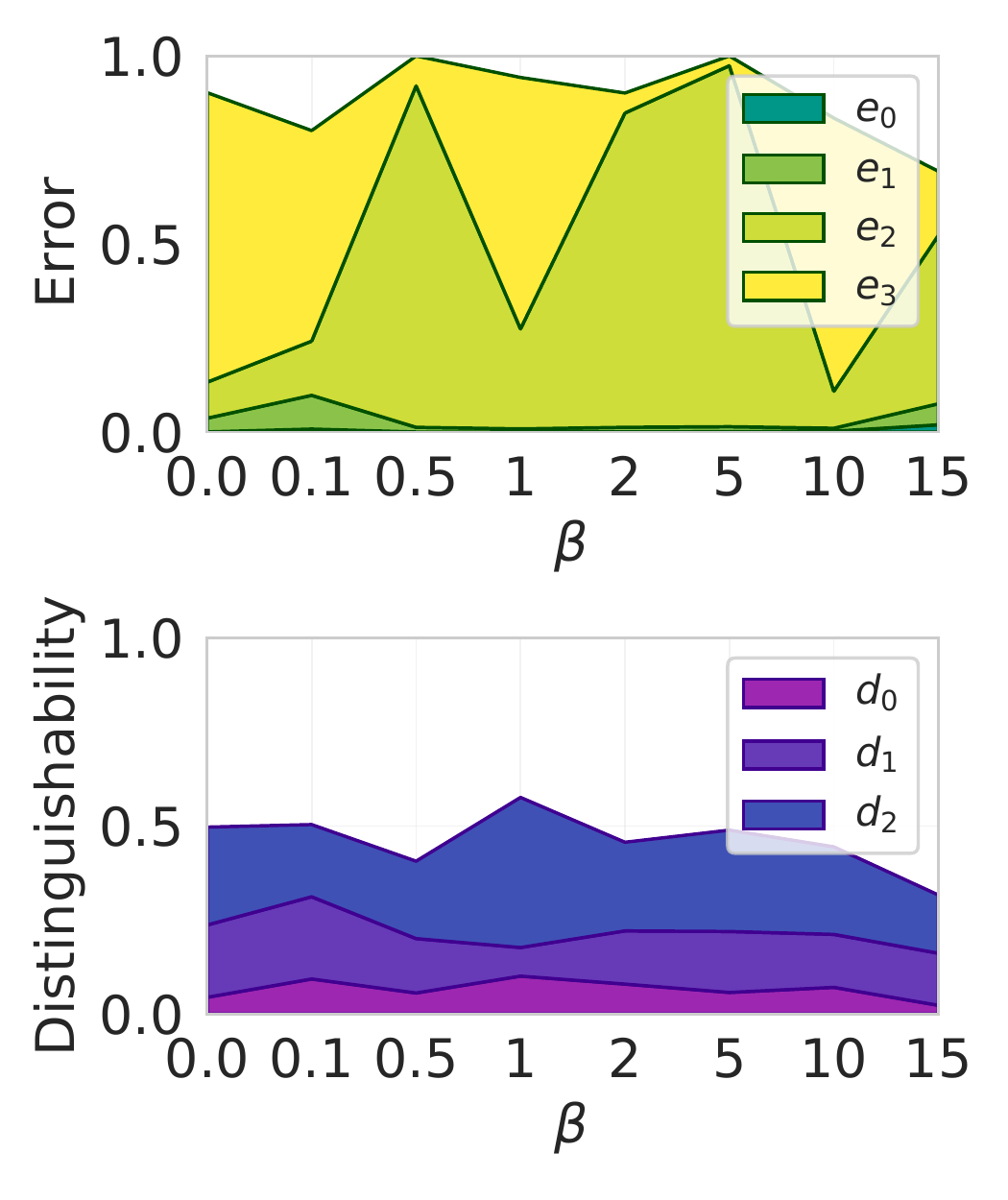}
    \caption{SD}
    \label{fig:CMNIST_SD}
\end{subfigure}
\begin{subfigure}{0.23\linewidth}
    \includegraphics[width=\linewidth]{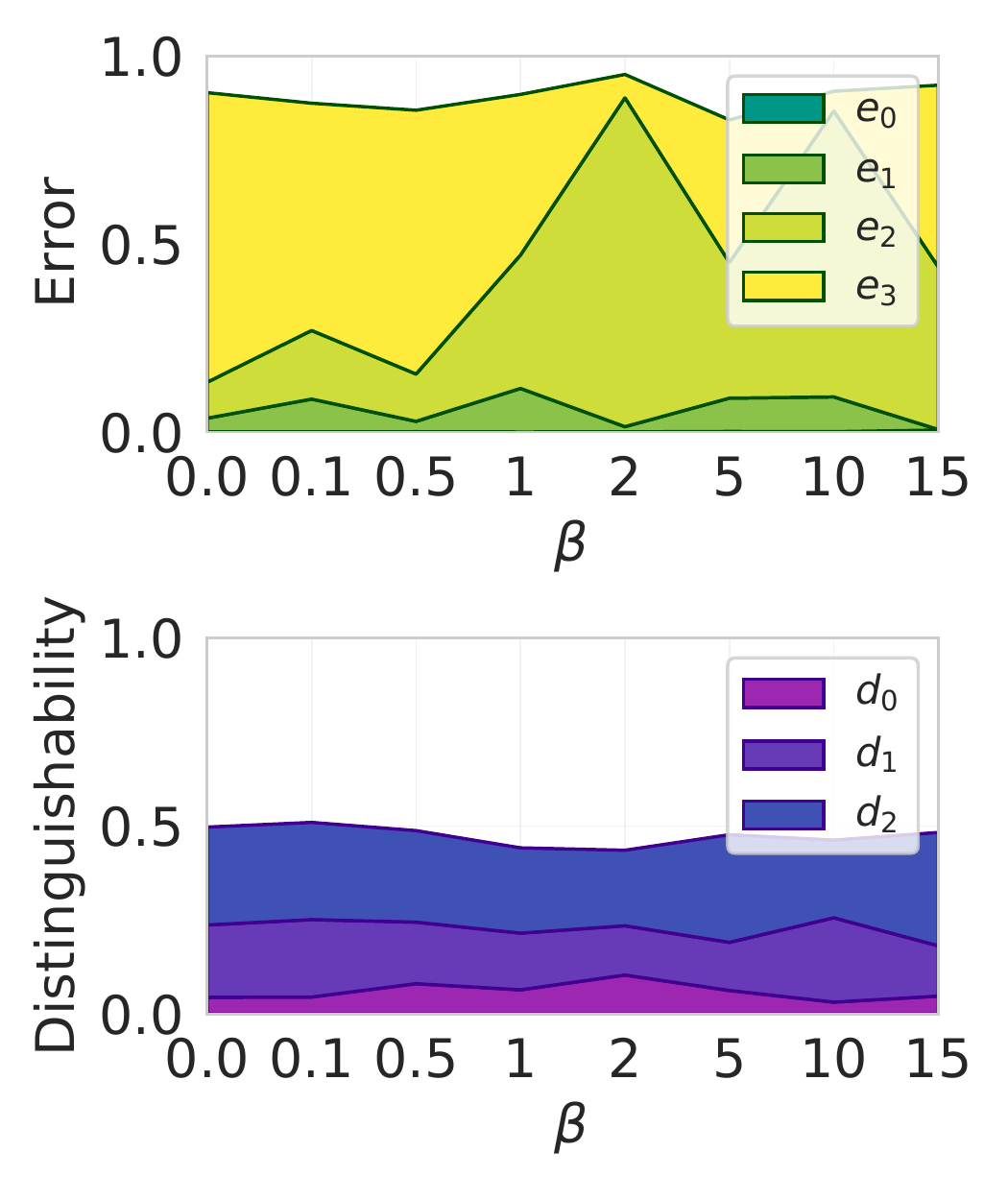}
    \caption{GroupDRO}
    \label{fig:CMNIST_GroupDRO}
\end{subfigure}
\begin{subfigure}{0.23\linewidth}
    \includegraphics[width=\linewidth]{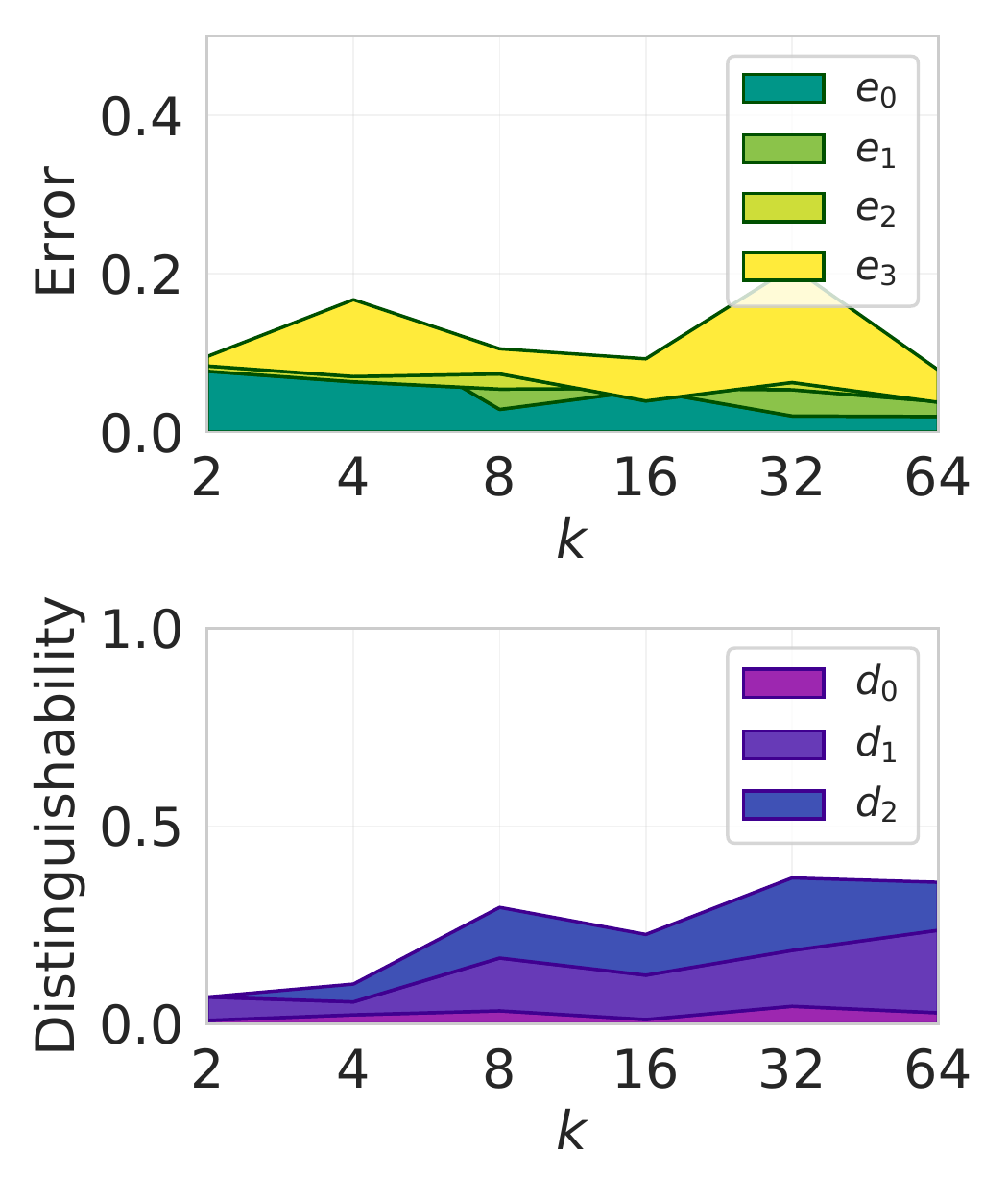}
    \caption{ERM+HSIC ($\beta=1$)}
    \label{fig:CMNIST_HSIC_resnetk}
\end{subfigure}

\caption{Decomposition of generalization errors and domain-distinguishability of three more algorithms on Colored MNIST dataset measured on the validation domains. Horizontal axis corresponds to regularization strength of the algorithms.}
\label{fig:CMNIST_additional}
\end{figure}

\begin{figure}
\centering
\begin{subfigure}{0.23\linewidth}
    \includegraphics[width=\linewidth]{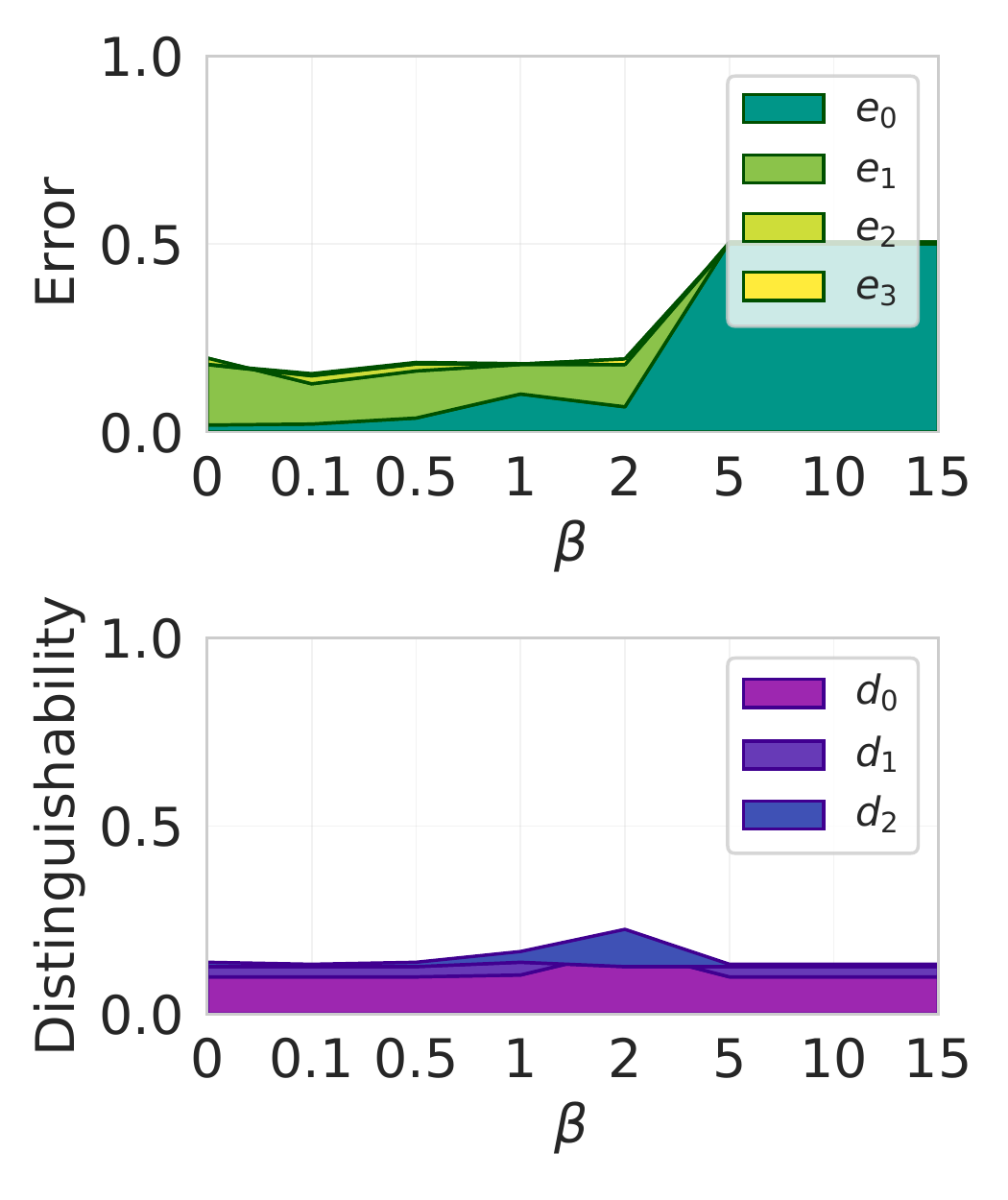}
    \caption{ERM+HSIC}
    \label{fig:Camelyon_HSIC}
\end{subfigure}
\begin{subfigure}{0.23\linewidth}
    \includegraphics[width=\linewidth]{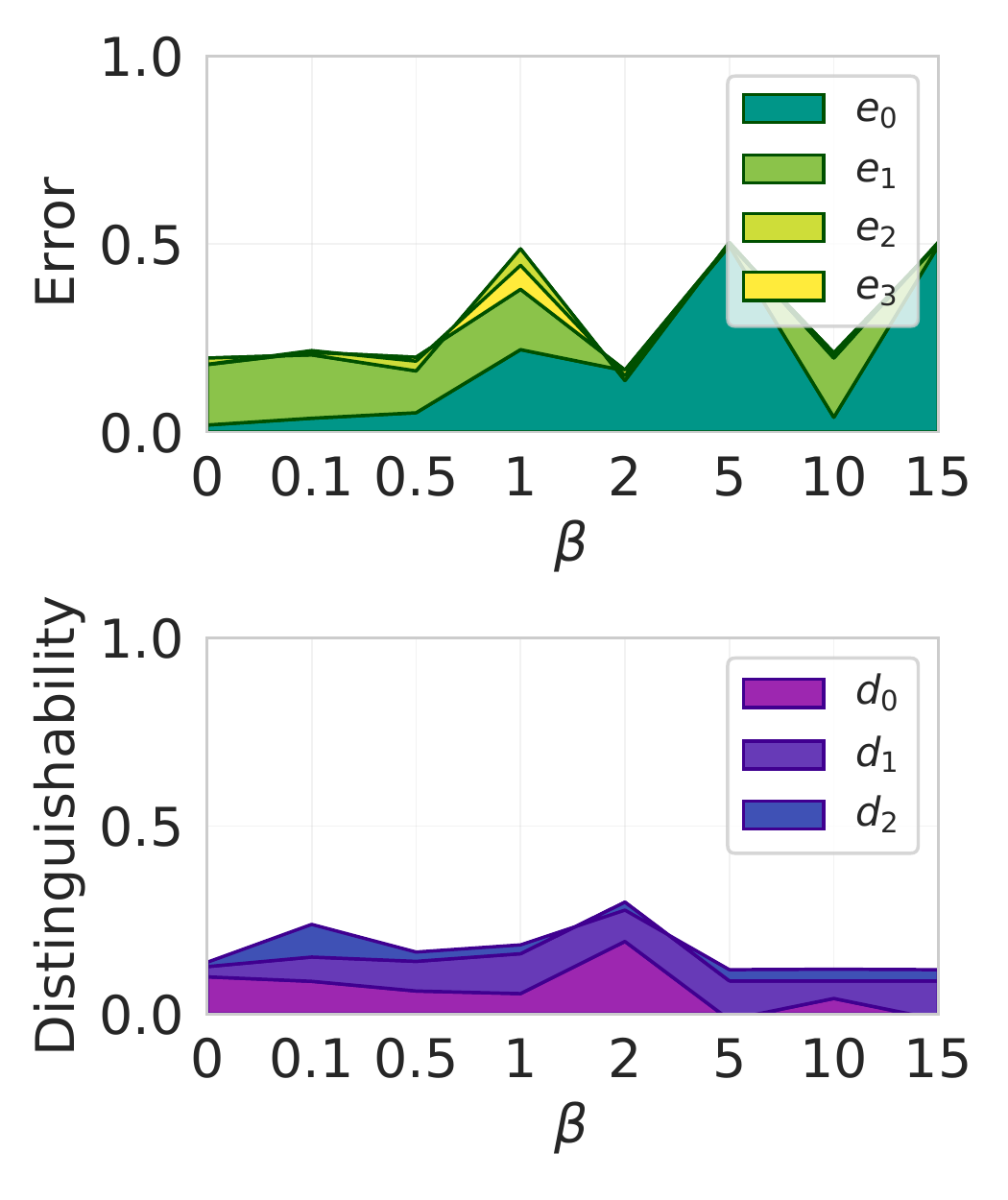}
    \caption{GroupDRO}
    \label{fig:Camelyon_DRO}
\end{subfigure}
\begin{subfigure}{0.23\linewidth}
    \includegraphics[width=\linewidth]{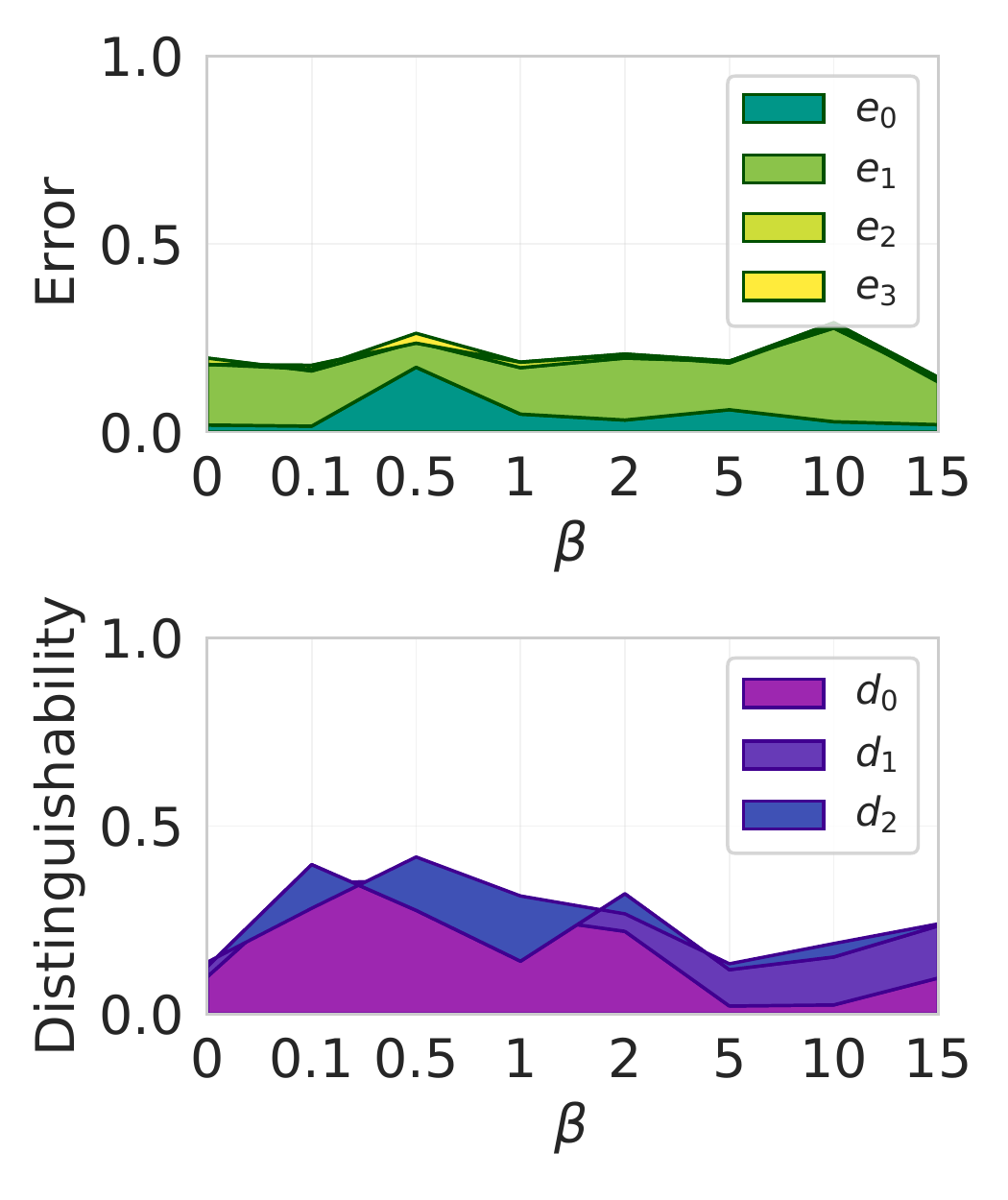}
    \caption{DeepCORAL}
    \label{fig:Camelyon_CORAL}
\end{subfigure}
\begin{subfigure}{0.23\linewidth}
    \includegraphics[width=\linewidth]{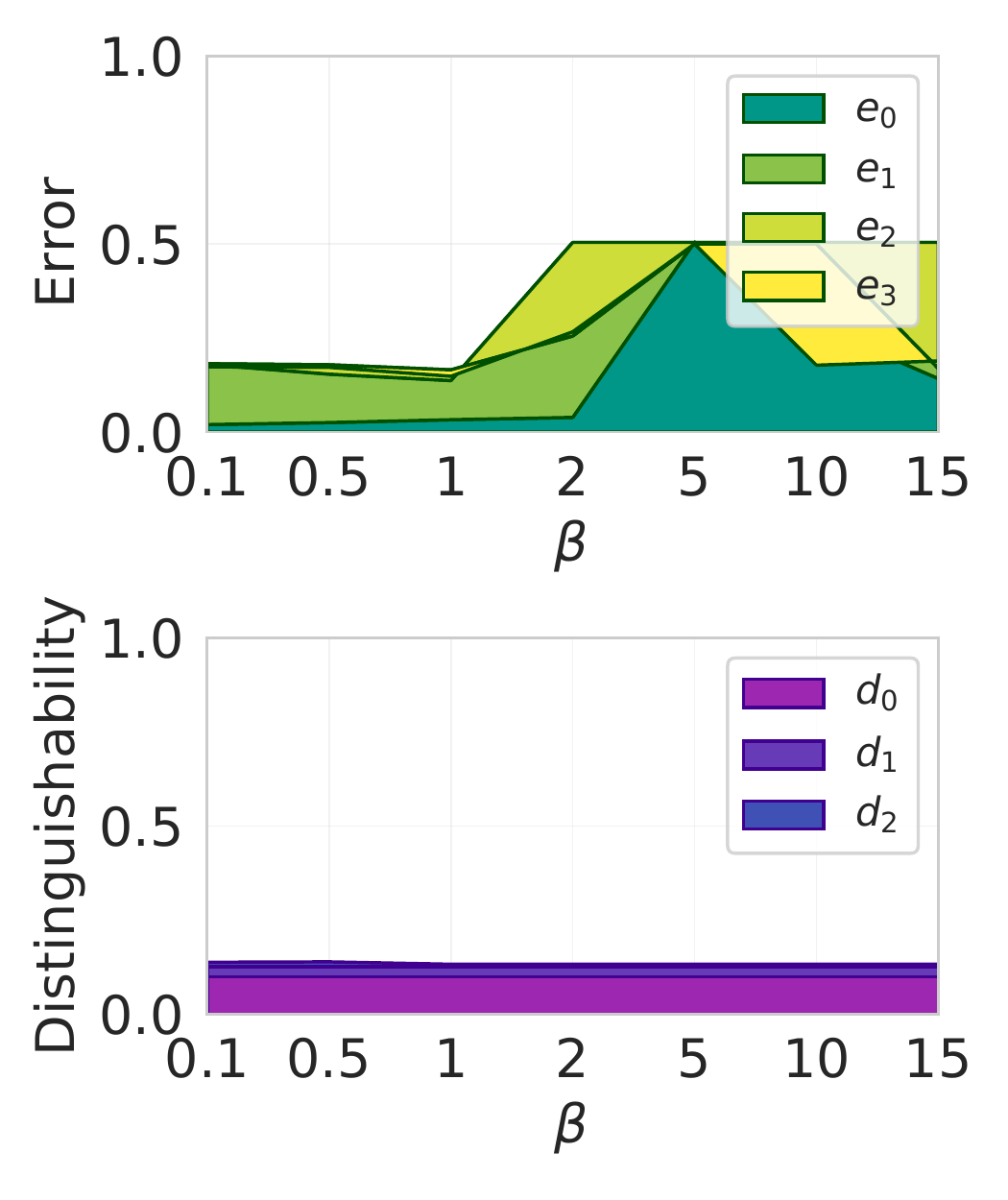}
    \caption{SD}
    \label{fig:Camelyon_SD}
\end{subfigure}
\caption{Decomposition of generalization errors and domain-distinguishability of three more algorithms on Camelyon17 dataset measured on the validation domains. Horizontal axis corresponds to regularization strength of the algorithms.}
\label{fig:Camelyon_CORAL_HSIC_DRO_SD}
\end{figure}

\begin{figure}
\centering
\begin{subfigure}{0.23\linewidth}
    \includegraphics[width=\linewidth]{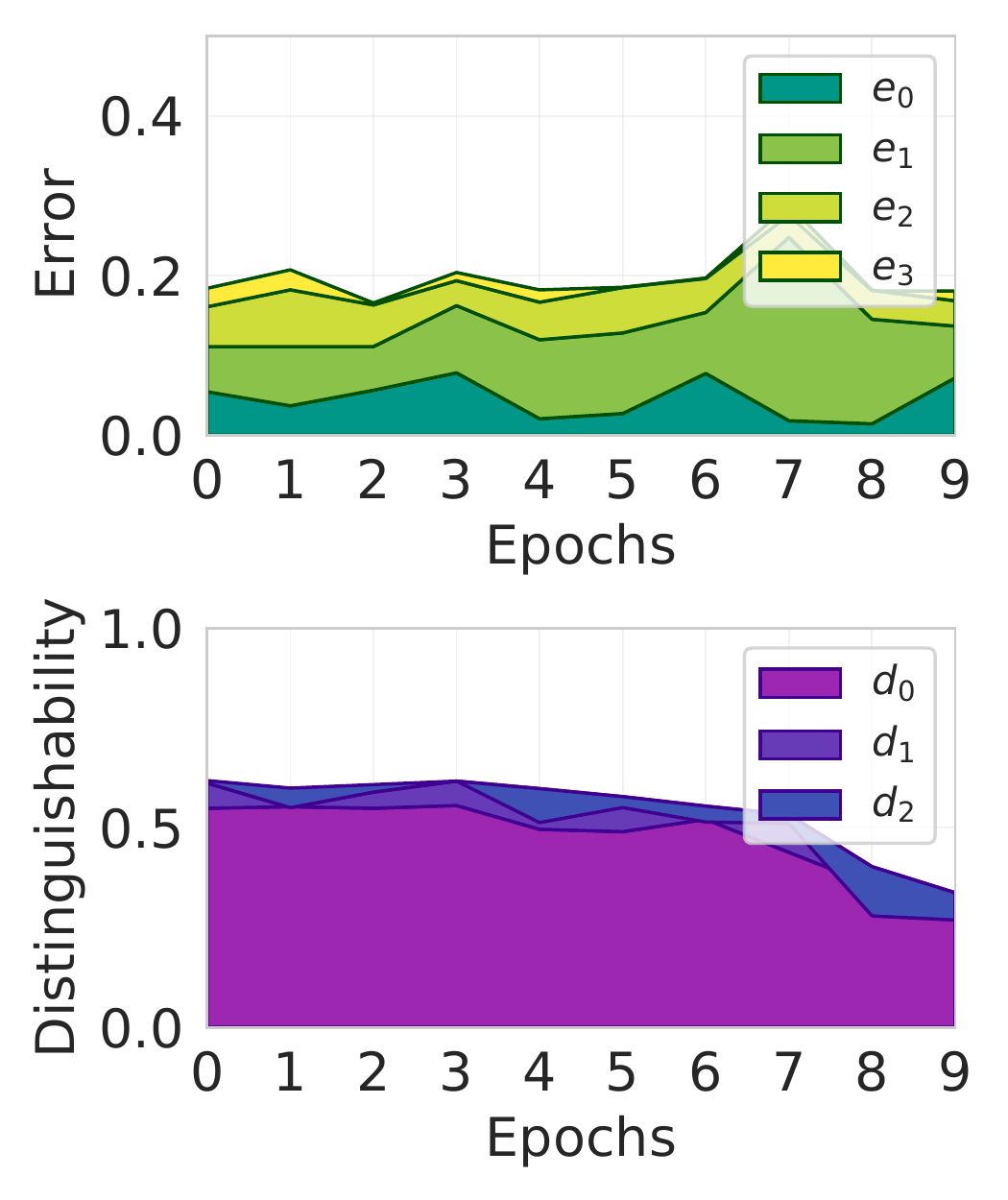}
    \caption{Camelyon17 (seed=1)}
    \label{fig:Camelyon_DeepCORAL_seed1}
\end{subfigure}
\begin{subfigure}{0.23\linewidth}
    \includegraphics[width=\linewidth]{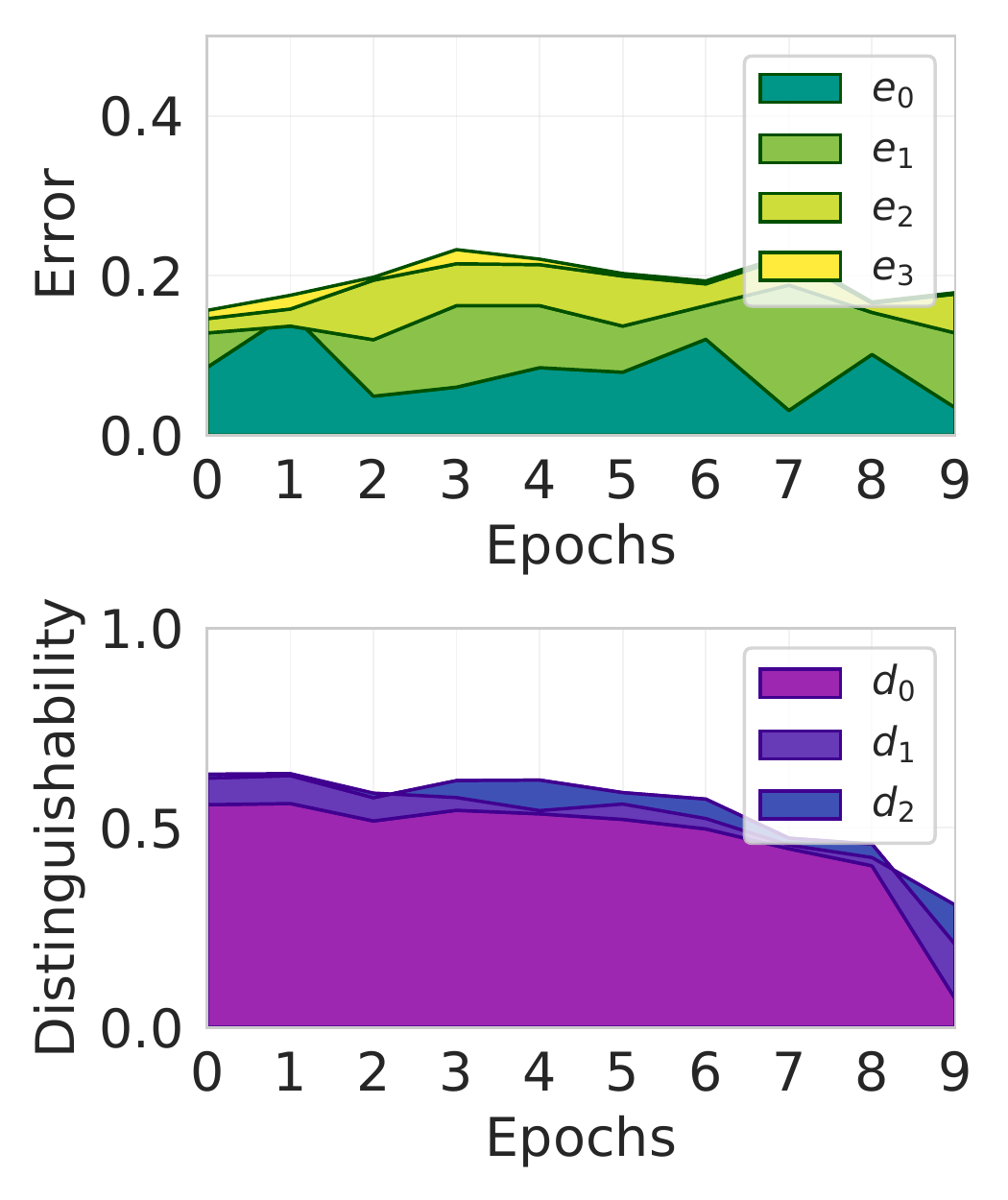}
    \caption{Camelyon17 (seed=2)}
    \label{fig:Camelyon_DeepCORAL_seed2}
\end{subfigure}
\begin{subfigure}{0.23\linewidth}
    \includegraphics[width=\linewidth]{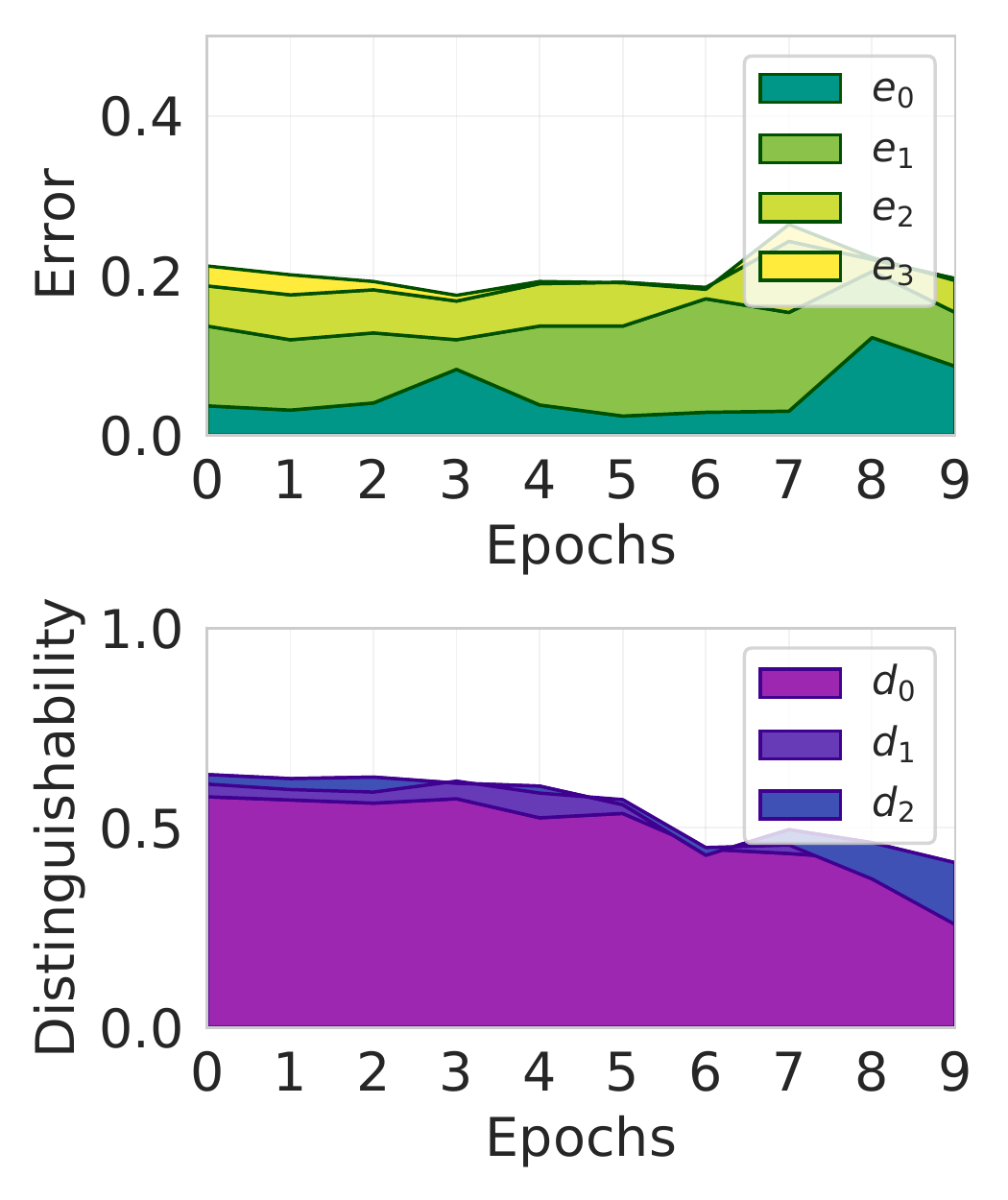}
    \caption{Camelyon17 (seed=3)}
    \label{fig:Camelyon_DeepCORAL_seed3}
\end{subfigure}
\caption{
The decomposition of generalization errors and domain-distinguishability of three DeepCORAL models trained on Camelyon17 ($\beta=5$) with different random seeds, measured on the validation domain. 
}
\label{fig:Camelyon_DeepCORAL_seed}
\end{figure}

\begin{figure}[!ht]
\centering
\begin{subfigure}{0.4\linewidth}
\includegraphics[width=\linewidth]{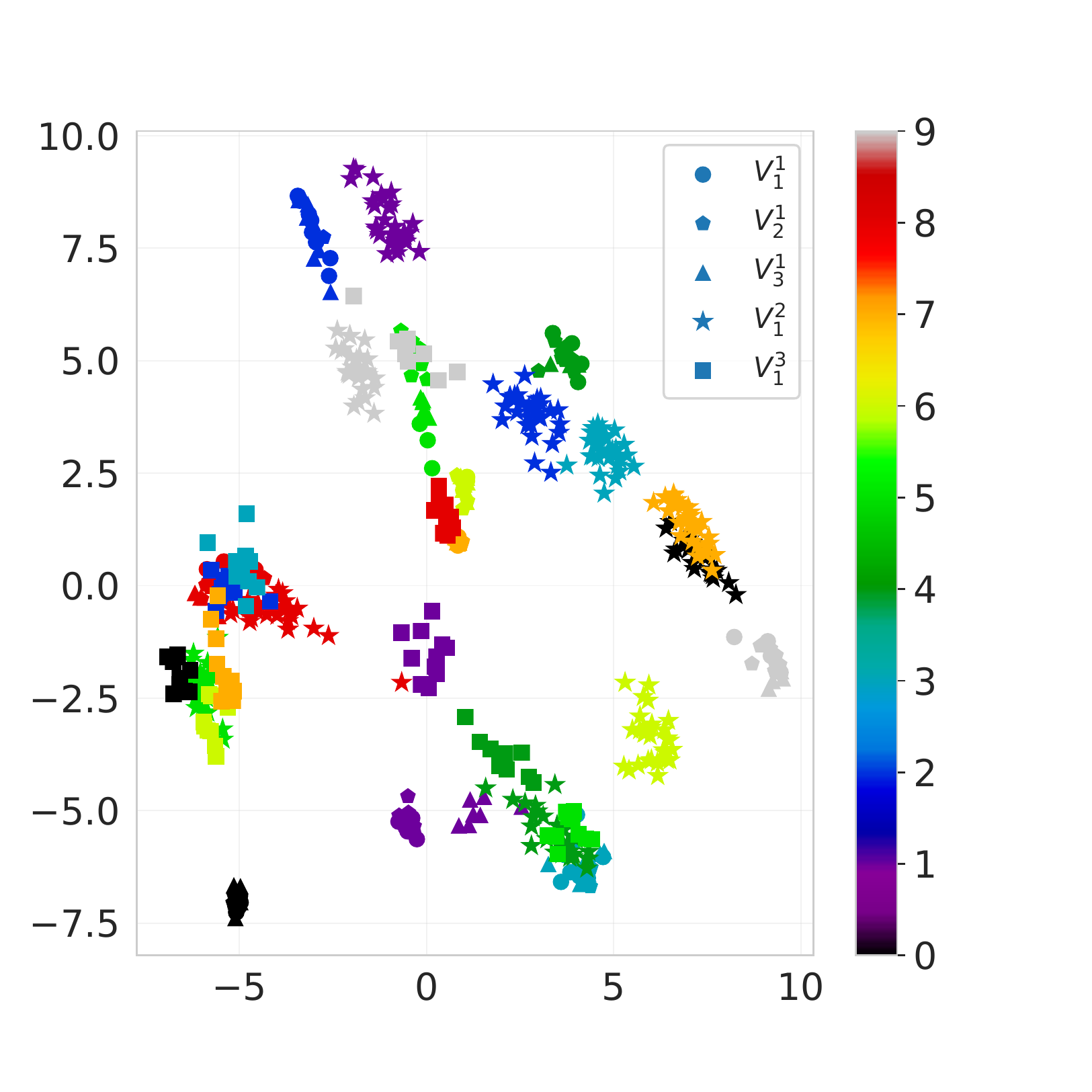}
    \caption{ERM on Colored MNIST}
    \label{fig:z_CMNIST_ERM}
\end{subfigure}
\begin{subfigure}{0.4\linewidth}
\includegraphics[width=\linewidth]{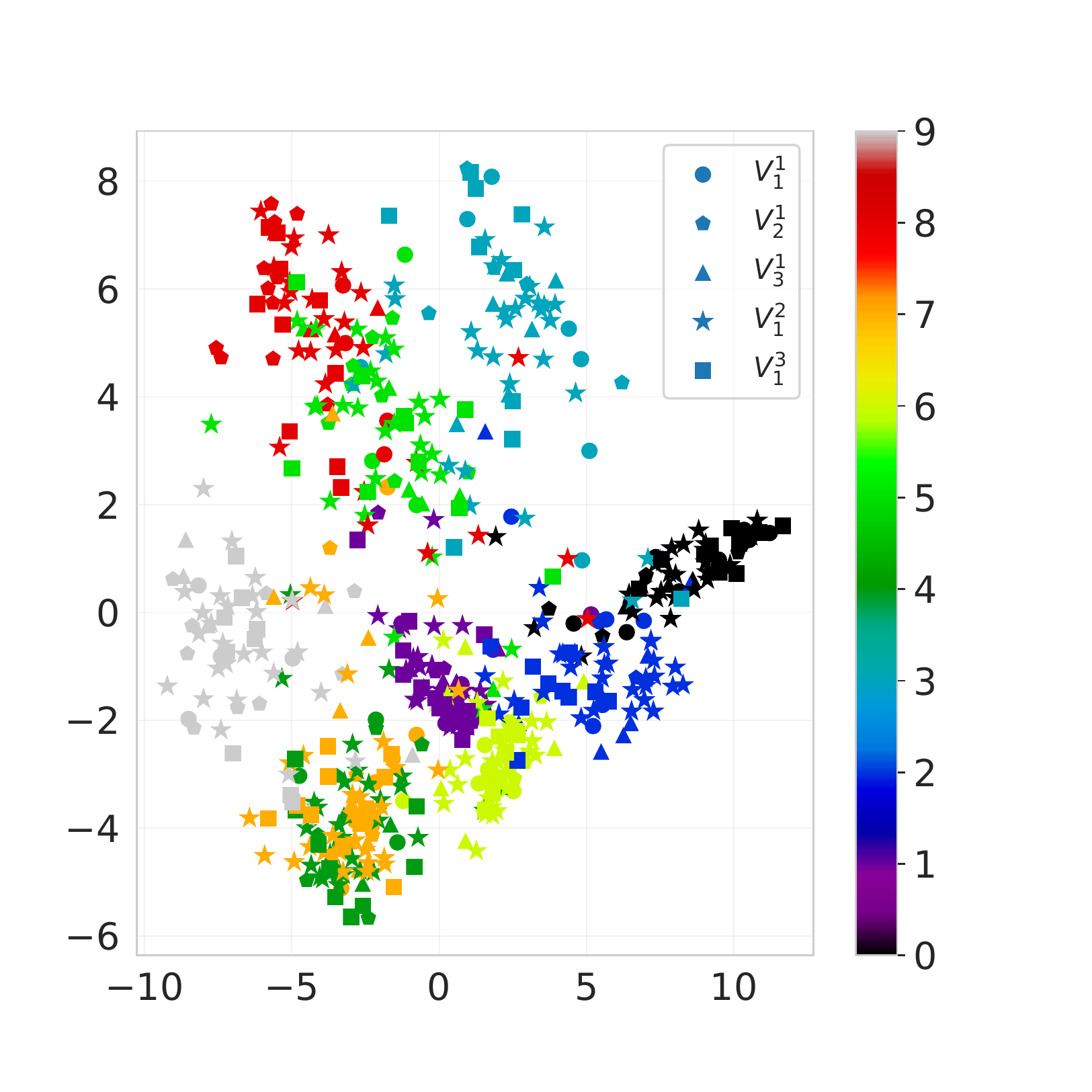}
    \caption{ERM+HSIC ($\beta=1$) on Colored MNIST}
    \label{fig:z_CMNIST_HSIC1}
\end{subfigure}

\begin{subfigure}{0.4\linewidth}
\includegraphics[width=\linewidth]{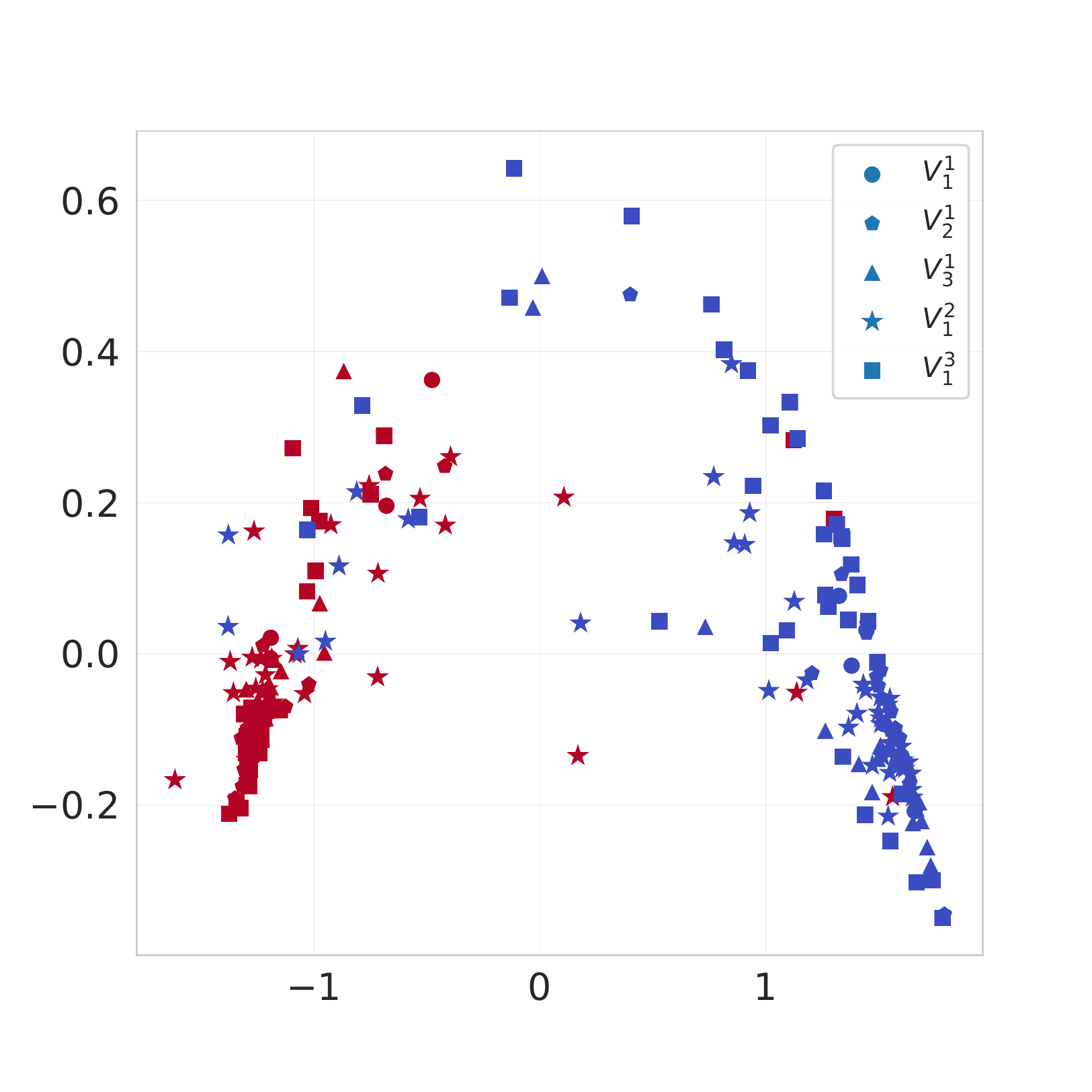}
    \caption{ERM on Camelyon17}
    \label{fig:z_Camelyon_ERM}
\end{subfigure}
\begin{subfigure}{0.4\linewidth}
\includegraphics[width=\linewidth]{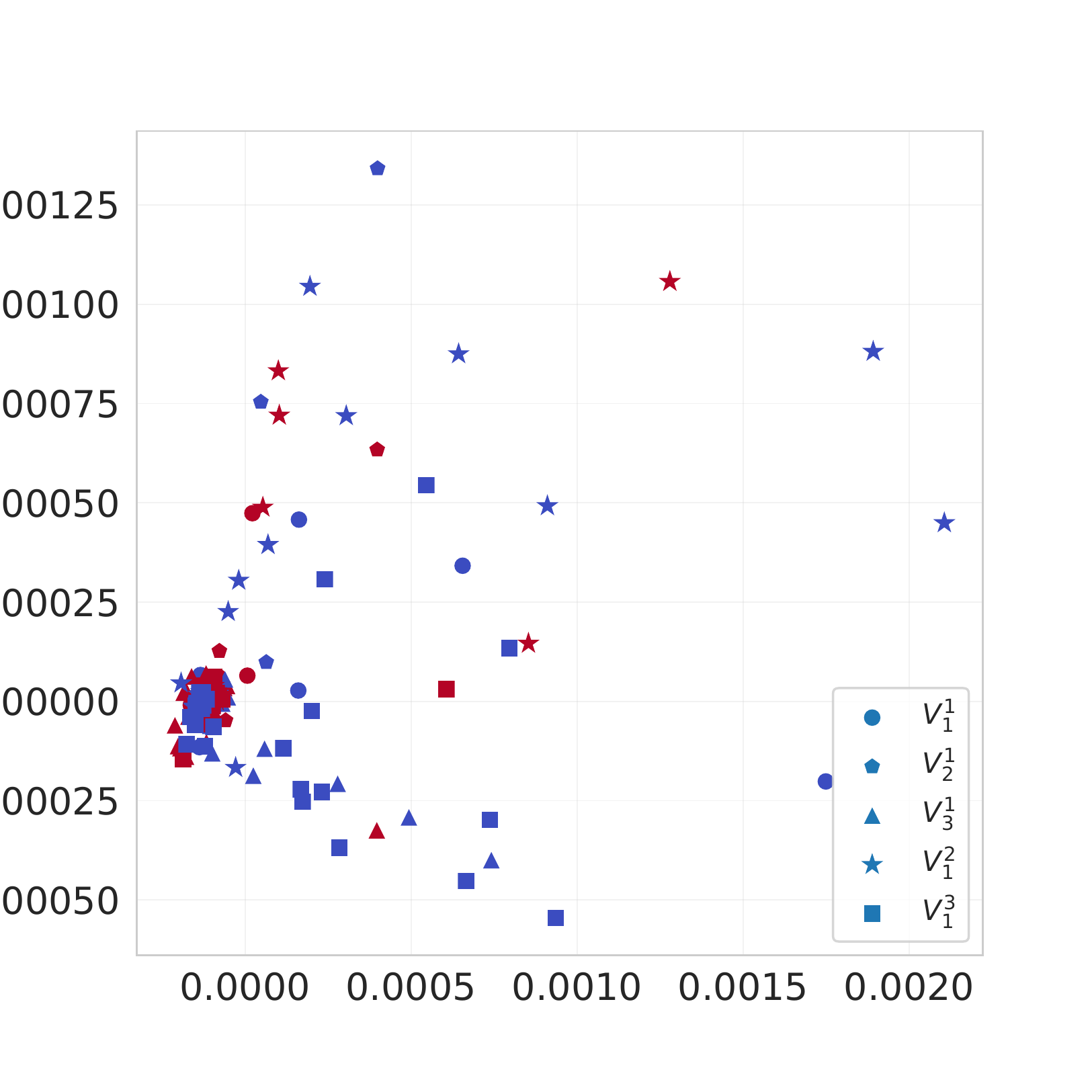}
    \caption{ERM+HSIC ($\beta=15$) on Camelyon17}
    \label{fig:z_Camelyon_HSIC15}
\end{subfigure}
\caption{2D PCA transformations of the learned representations of several models (from the last epoch). Colors indicate the labels, while marker shapes indicate the domains.}
\label{fig:z-space}
\end{figure}

